\documentclass{article}

\usepackage[preprint]{neurips_2026}

\usepackage[utf8]{inputenc} 
\usepackage[T1]{fontenc}    
\usepackage{hyperref}       
\usepackage{url}            
\usepackage{booktabs}       
\usepackage{amsfonts}       
\usepackage{nicefrac}       
\usepackage{microtype}      
\usepackage{xcolor}         

\usepackage{graphicx}
\usepackage{float}
\usepackage{amsmath,amsthm,amssymb}
\usepackage{mathtools}
\usepackage{bm,mathabx}
\usepackage{pgfplots} 
\usepackage{pgfplotstable} 
\usepackage{subcaption}
\usepgfplotslibrary{colorbrewer}
\pgfplotsset{compat=1.18}
\usepackage[normalem]{ulem}
\usepackage{pdflscape}
\usepackage{wrapfig}
\usepackage{enumitem}
\usepackage[capitalize,noabbrev]{cleveref}

\usepackage{tikz}

\usetikzlibrary{arrows.meta}     
\usetikzlibrary{positioning}     
\usetikzlibrary{calc}            
\usetikzlibrary{shapes.geometric} 

\Crefname{equation}{Eq.}{Eqs.}
\Crefname{align}{Eq.}{Eqs.}
\Crefname{figure}{Fig.}{Figs.}
\Crefname{tabular}{Tab.}{Tabs.}
\Crefname{theorem}{Thm.}{Thms.}
\Crefname{lemma}{Lem.}{Lems.}
\Crefname{proposition}{Prop.}{Props.}
\Crefname{definition}{Def.}{Defs.}
\Crefname{algorithm}{Algo.}{Algs.}
\Crefname{corollary}{Corol.}{Corol.}
\Crefname{section}{Sec.}{Sec.}
\Crefname{appendix}{Appendix}{Appendices}
\Crefname{prop}{Prop.}{Props.}
\Crefname{task}{Task.}{Tasks.}
\Crefname{setting}{Setting.}{Settings.}
\Crefname{example}{Example}{Examples}

\renewcommand{\div}{\nabla\cdot}

\newcommand{\mr}{\mathrm}

\newtheorem{theorem}{Theorem}[section]

\newtheorem{lemma}{Lemma}[section]

\newtheorem{proposition}{Proposition}[section]

\theoremstyle{definition}
\newtheorem{example}{Example}[section]

\title{Lagrangian Flow Matching: A Least-Action Framework for Principled Path Design}
\author{%
  Shukai Du\textsuperscript{1}\thanks{Equal contribution.}
  \quad Junzhe Zhang\textsuperscript{2}\footnotemark[1]
  \quad Yiming Li\textsuperscript{1} \\
  \textsuperscript{1}Department of Mathematics, Syracuse University \\
  \textsuperscript{2}Department of EECS, Syracuse University \\
  \texttt{\{sdu113,jzhan403,yli658\}@syr.edu}
}

\begin{document}

\maketitle

\begin{abstract}
Flow matching trains a neural velocity field by regression against
a target velocity associated with a prescribed probability path
connecting a simple initial distribution to the data distribution.
A central design choice is the path itself. Existing constructions,
including rectified and optimal-transport-based paths, transport
samples along straight lines between coupled endpoints and thus
cover only a narrow class of dynamics. We observe that this corresponds to the simplest case of the \emph{least-action principle} in classical mechanics, in which the kinetic Lagrangian $\mathcal{L}=\tfrac12\|\dot x\|^2$ yields free-particle straight-line trajectories. Building on this observation, we propose \emph{Lagrangian flow matching}, a physics-based framework in which the probability path and velocity field are determined by minimizing the action of a general Lagrangian subject to the continuity equation and the prescribed endpoints. We show that this dynamic problem admits an equivalent static optimal transport (OT) formulation, yielding a family of simulation-free training objectives that recover OT-based flow matching as the kinetic special case and the trigonometric variance-preserving diffusion path as the harmonic-oscillator case. More general Lagrangians give rise to new probability paths and velocity fields, and numerical experiments show that they induce meaningful changes in the learned dynamics while remaining competitive with existing conditional flow matching models.
\end{abstract}
\section{Introduction}
\label{sec:_intro}

A natural way to approach generative modeling is from a continuous-time
perspective: rather than modeling the data distribution directly, one
constructs a time-dependent probability path $(p_t)_{t\in[0,1]}$ that
continuously transports a simple initial distribution $p_{\mathrm{init}}$
into the target distribution $p_{\mathrm{data}}$. Two important classes
of generative models built on this perspective are diffusion models
\cite{sohl2015deep,ho2020denoising,song2021scorebased} and flow-based
models
\cite{chen2018neural,grathwohl2018scalable,lipman2023flow,liu2023flow,tong2024improving},
which realize the transport through stochastic differential equations (SDE)
and ordinary differential equations (ODE), respectively.

We focus on flow-based models. An influential early class is continuous
normalizing flows (CNFs)
\cite{chen2018neural,grathwohl2018scalable}, which are expressive enough
to represent a broad class of continuous probability paths
\cite{song2021maximum} but are computationally expensive to train:
likelihood-based training requires repeated forward and backward ODE
solves. \emph{Flow matching} provides an efficient alternative by
directly regressing a neural velocity field $v_\theta(x,t)$ against a
target velocity field $v_t(x)$ associated with a \emph{prescribed
probability path} \cite{lipman2023flow,liu2023flow,tong2024improving}.
By avoiding ODE solves during training, this simulation-free objective
has enabled flow-based models to scale to high-dimensional data and
large-scale applications.

The performance of a flow matching model, however, is determined not
only by the regression objective itself but also by the prescribed
probability path $(p_t)$ and its associated velocity field $v_t$.
Existing methods choose these objects from a relatively narrow design
space: conditional Gaussian paths
\cite{lipman2023flow,tong2024improving}, affine interpolations between
coupled endpoints \cite{liu2023flow}, and optimal-transport-based
straight-line interpolations \cite{tong2024improving}. 
This raises a natural question: can the
probability paths and velocity fields used in flow matching be selected
from a more general, principled design space?

\begin{wrapfigure}[21]{r}{0.71\linewidth}
\vspace{-1.2\baselineskip}
\centering
\resizebox{\linewidth}{!}{%
\begin{tikzpicture}[
>=Stealth,
font=\small,
collabel/.style={font=\bfseries, align=center},
rowlabel/.style={font=\bfseries, align=center, text=black!75},
celllabel/.style={align=center, font=\footnotesize\itshape},
note/.style={align=center, font=\footnotesize, text=black!65},
bridgenote/.style={align=left, font=\footnotesize}
]
\fill[blue!4]   (0.04, 0.04) rectangle (3.96, 4.56);
\fill[purple!4] (4.04, 0.04) rectangle (7.96, 4.56);
\draw[thick, rounded corners=3pt, black!55]
  (0,0) rectangle (8,4.6);
\draw[thick, black!40] (4,0) -- (4,4.6);
\draw[thick, black!40] (0,2.3) -- (8,2.3);
\node[collabel] at (2,5.1)
{\small Kinetic Lagrangian \\[2pt]
{\normalfont\small $\mathcal{L} = \tfrac12\|\dot x\|^2$}};
\node[collabel] at (6,5.1)
{\small General Lagrangian \\[2pt]
{\normalfont\small $\mathcal{L}(x,\dot x,t) = K(\dot x) - V(x)$}};
\node[rowlabel, rotate=90] at (-0.55,3.45)
{\small Deterministic \\ trajectory};
\node[rowlabel, rotate=90] at (-0.55,1.15)
{\small Probability \\ path};
\node[celllabel, text=blue!55!black] at (2,4.20)
  {free-particle trajectory};
\node at (2,3.4)
  {\includegraphics[width=3cm]{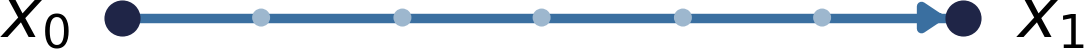}};
\node[note] at (2,2.60) {$\gamma(t)=(1-t)x_0+t x_1$};
\node[celllabel, text=purple!65!black] at (6,4.20)
  {least-action trajectory};
\node at (6,3.4)
  {\includegraphics[width=3cm]{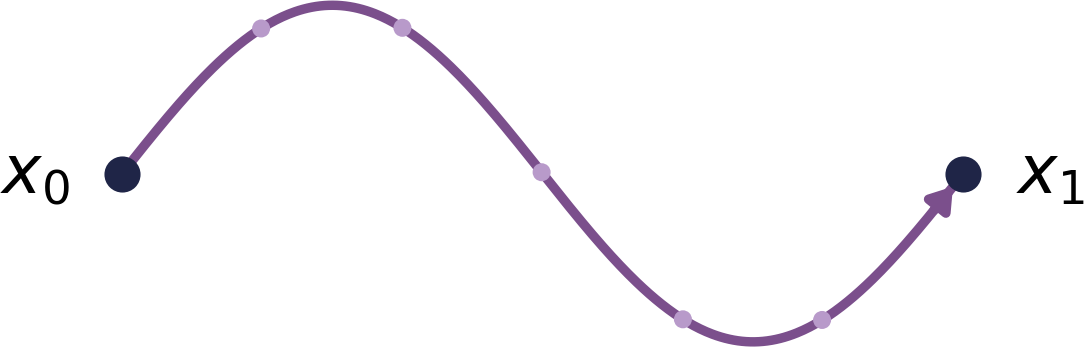}};
\node[note] at (6,2.60)
  {$\partial\mathcal{L}/\partial x = \tfrac{d}{dt}\,\partial\mathcal{L}/\partial \dot x$};
\node[celllabel, text=blue!55!black] at (2,1.90)
  {OT-based flow matching};
\node at (2,0.9)
  {\includegraphics[width=3cm]{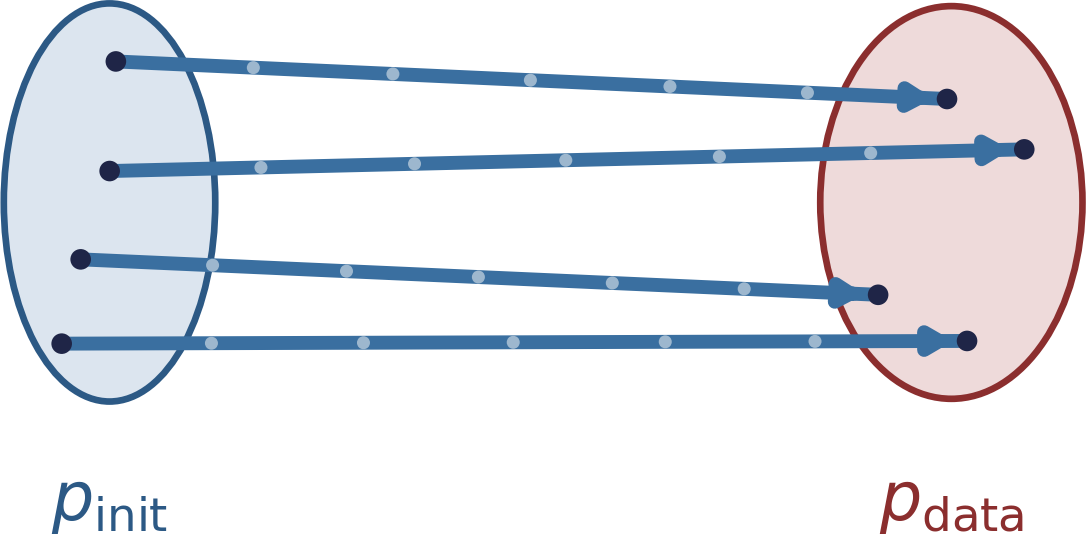}};
\node[celllabel, text=purple!65!black] at (6,1.90)
  {Lagrangian flow matching};
\node at (6,0.9)
  {\includegraphics[width=3cm]{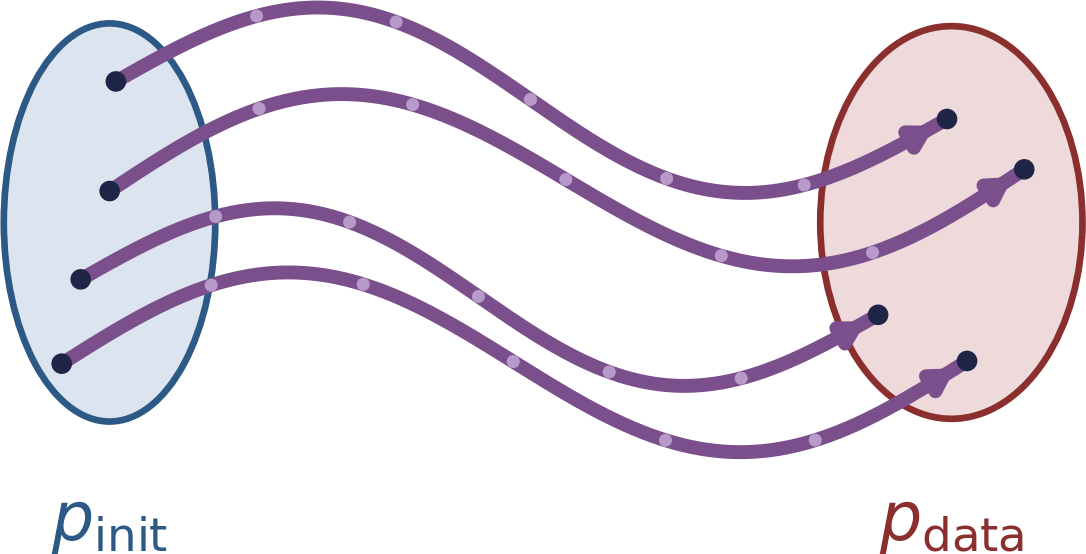}};
\draw[decorate, decoration={brace, amplitude=5pt},
      thick, black!55]
  (8.1,3.45) -- (8.1,1.15);
\node[rowlabel, rotate=270] at (8.7,2.4)
{\small Static--Dynamic \\ Equivalence {\normalfont\small (Thm.~3.1)}};
\end{tikzpicture}%
}
\caption{Lagrangian flow matching unifies trajectory selection (top row) and probability-path construction (bottom row) under the least-action principle, linked by the static--dynamic equivalence (\Cref{thm:connection_pi_v_p}). The kinetic Lagrangian (left column) yields straight-line trajectories and recovers OT-based flow matching; a general Lagrangian (right column) yields curved trajectories and induces new probability paths.}
\label{fig:lagrangian_flow_table}
\end{wrapfigure}
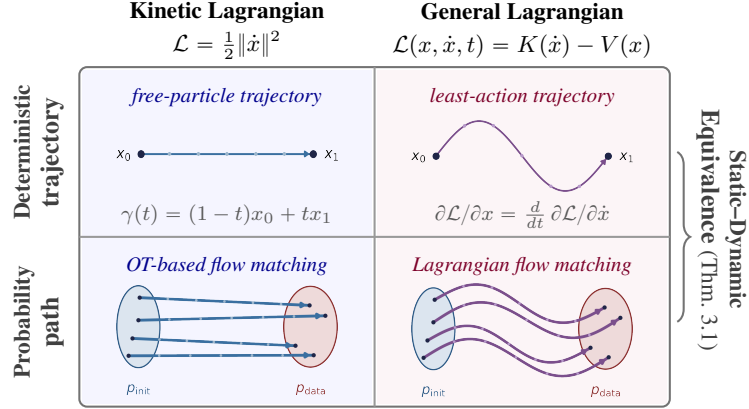
We address this question through the lens of the
least-action principle, a foundational variational principle in classical
mechanics~\citep{arnold1989mathematical}. Given a Lagrangian $\mathcal{L}(x, \dot x, t)$, the principle selects, between two fixed endpoints, the trajectory that minimizes the corresponding action. The choice of Lagrangian fully determines the dynamics: the kinetic case yields straight-line free-particle trajectories (\Cref{fig:lagrangian_flow_table}, top-left), while a 
general Lagrangian 
produces curved paths governed by the Euler--Lagrange equations (\Cref{fig:lagrangian_flow_table}, top-right).

The same principle applies naturally to flow matching. Existing construction based on affine and OT-based paths transport samples
along straight-line trajectories between coupled endpoints
(\Cref{fig:lagrangian_flow_table}, bottom-left), corresponding to the
kinetic case. Replacing the kinetic Lagrangian by a more general one
yields curved least-action trajectories between matched endpoints,
which in turn induce new probability paths
(\Cref{fig:lagrangian_flow_table}, bottom-right). This transfer from trajectories to probability paths are made rigorous by the classical
static--dynamic equivalence in optimal transport
\cite{benamou2000computational,bernard2007optimal,villani2009optimal}.

Several recent works have also explored variational or least-action
formulations for learning transport dynamics
\cite{neklyudov2023action,neklyudov2024a,pooladian2024neural,koshizuka2022neural}.
Comparatively little work, however, has aimed to cast least-action dynamics directly as a flow-matching problem. 
We aim to fill this gap by introducing a
Lagrangian perspective to flow matching: the Lagrangian defines endpoint costs,
selects couplings, and constructs explicit least-action paths. These paths are
then used to define a standard flow-matching regression objective. In this sense,
our approach aims to combine the flexibility of least-action-induced probability
paths with the simplicity and efficiency of the simulation-free regression in flow matching.

Building on this perspective, we propose \emph{Lagrangian flow
matching}, a framework in which the probability path and velocity
field are determined by minimizing a Lagrangian action subject to the
continuity equation and the endpoint constraints
$p_0 = p_{\mathrm{init}}$ and $p_1 = p_{\mathrm{data}}$. Our
contributions are threefold. First, we propose a physics-based
framework for selecting flow matching probability paths and velocity
fields through a least-action principle. Second, we derive a family
of simulation-free training objectives, and show that several
existing constructions, including OT-based and conditional flow
matching models, are recovered as special cases. Third, we derive new
flow matching models from more general Lagrangians, in particular a
\emph{harmonic family} indexed by a frequency $\omega \in (0, \pi)$
that continuously deforms the 
least-action path,
and recovers OT-based flow matching and the trigonometric variance-preserving schedule at its two endpoints. We validate the framework on synthetic
2D data, single-cell trajectory interpolation, and CIFAR-10 image
generation, and show that different Lagrangians induce meaningful
changes in the learned dynamics while remaining competitive with
existing conditional flow matching models. All proofs of the results and details about the experimental setup are provided in \Cref{app:_b,app:exp_setup}. Code is available at \url{https://github.com/junzhez/lagrangian-flow-matching}.

\vspace{-0.5\baselineskip}
\subsection{Preliminaries and Related Work}
\label{sec:_prelim}
\vspace{-0.5\baselineskip}

Throughout, $p_{\mathrm{init}}$ and $p_{\mathrm{data}}$ are
probability densities on $\mathbb R^d$. The goal of generative
modeling is to learn a transformation pushing $p_{\mathrm{init}}$ to
$p_{\mathrm{data}}$, given finitely many samples from
$p_{\mathrm{data}}$. This setting includes both the standard case in
which $p_{\mathrm{init}}$ is easily sampled (e.g., a standard
Gaussian) and the more general case in which both
$p_{\mathrm{init}}$ and $p_{\mathrm{data}}$ are accessed only through
samples \cite{albergo2023building,liu2023flow,tong2024improving}. We
focus on continuous-time approaches realizing this transformation
through a probability path $(p_t)_{t\in[0,1]}$ with
$p_0=p_{\mathrm{init}}$ and $p_1=p_{\mathrm{data}}$.
 
\textbf{Flow matching} trains a parametrized velocity field
$v_\theta(x,t)$ by regressing against a target field
$v_t(x)$:
\begin{align}
\label{eq:flow_match_obj}
\min_\theta\;
\mathbb E_{t\sim \mathcal U[0,1],\,x\sim p_t}
\bigl\|v_\theta(x,t)-v_t(x)\bigr\|^2,
\end{align}
where $(p_t)_{t\in[0,1]}$ satisfies the endpoint conditions
$p_0=p_{\mathrm{init}}$, $p_1=p_{\mathrm{data}}$, and $v_t$ generates
this path through the continuity equation
$\partial_t p_t+\nabla\cdot(p_t v_t)=0$. Infinitely many pairs
$(p_t,v_t)$ satisfy these conditions, and the model's behavior
depends on the choice. 
 
\textbf{Conditional flow matching} addresses this intractability by
expressing $p_t$ as a mixture
$p_t(x)=\mathbb E_{z\sim q}[p_t(x\mid z)]$ over a conditioning
variable $z\sim q$ \cite{lipman2023flow,tong2024improving}. If
$v_t(x\mid z)$ generates $p_t(\cdot\mid z)$, then the marginal
velocity is the posterior average
$v_t(x)=\mathbb E[v_t(x\mid z)\mid x_t=x]$. The marginal objective
\eqref{eq:flow_match_obj} is equal, up to an additive constant, to
\begin{align}
\label{eq:cond_flow_match_obj}
\min_\theta\;
\mathbb E_{t\sim \mathcal U[0,1],\,z\sim q,\,x\sim p_t(\cdot\mid z)}
\bigl\|v_\theta(x,t)-v_t(x\mid z)\bigr\|^2,
\end{align}
which is tractable when the conditional paths and velocities admit
closed-form expressions. A representative example is the affine
construction of \cite{lipman2023flow}: with $z=y\sim p_{\mathrm{data}}$
and $p_{\mathrm{init}}=\mathcal N(0,I)$, the conditional trajectory
is $\psi_t(x\mid y)=(1-(1-\varepsilon)t)x+ty$. We recover this
construction as the $\omega\to 0$ limit of conditional Lagrangian
flow matching (\Cref{ex:harmonic_to_ot}).
 
\textbf{Dynamic optimal transport.} The least-action perspective
developed here builds on the dynamic formulation of optimal transport
due to Benamou and Brenier \cite{benamou2000computational}. The
squared $2$-Wasserstein distance between $p_0$ and $p_1$ admits the
representation
\begin{align}
\label{eq:dyn_opt_transp}
W_2^2(p_0,p_1)
=
\inf_{(p_t,v_t)}
\int_0^1\!\!\int_{\mathbb R^d}
\|v_t(x)\|^2\,p_t(x)\,dx\,dt,
\end{align}
where the infimum is over pairs $(p_t,v_t)$ satisfying the continuity
equation and the endpoint constraints $p_{t=0}=p_0$, $p_{t=1}=p_1$.
This is the kinetic special case of the Lagrangian framework
developed below, and the basic viewpoint behind OT-based flow
matching \cite{tong2024improving}. We defer readers to  \Cref{app:related_work} for a more detailed discussion on related work.
\section{Lagrangian Flow Matching}
\label{sec:lagrangian_flow_matching}

This section introduces our main framework. We first recall the
elements of Lagrangian mechanics needed to state a least-action
principle, then formulate the \emph{Lagrangian optimal transport}
problem and establish the static--dynamic equivalence that converts
it into a tractable training objective
(\Cref{sec:lagrangian_ot}). The conditional formulation, which makes
the framework practical at scale, is developed in
\Cref{sec:conditional_lfm}.

\begin{wrapfigure}[11]{r}{0.5\linewidth}
\vspace{-0.2in}
\centering
\includegraphics[width=\linewidth]{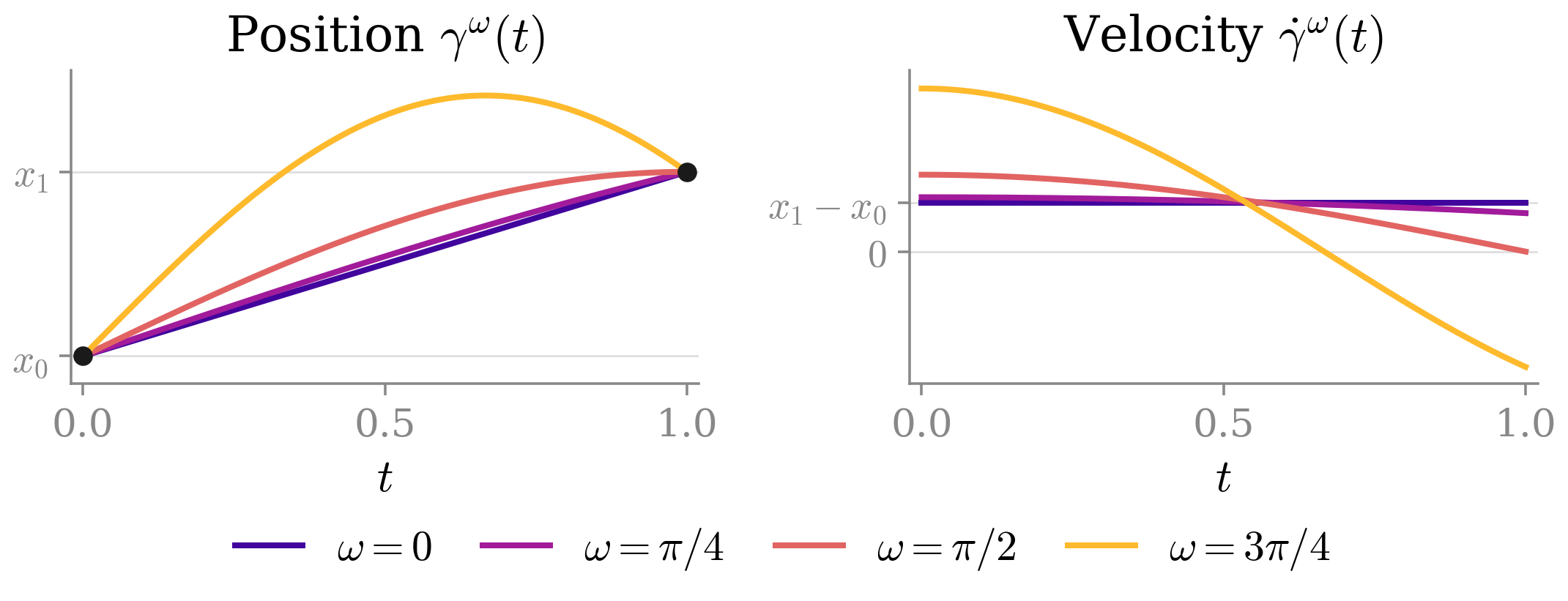}
\caption{Position $\gamma^\omega(t)$ and velocity
$\dot\gamma^\omega(t)$ of the harmonic least-action trajectory
\eqref{eq:harmonic_gamma} for $\omega\in(0,\pi)$, with endpoints
$x_0=0$ and $x_1=1$.}
\label{fig:harmonic_pos_vel}
\end{wrapfigure}

Lagrangian mechanics provides a variational principle for selecting
dynamics through an action functional, and has long served as a
unifying language for describing physical
systems~\citep{arnold1989mathematical}. A \emph{Lagrangian} is a
function $\mathcal L(x,v,t)$ of position $x\in\mathbb R^d$, velocity
$v\in\mathbb R^d$, and time $t\in[0,1]$ encoding the dynamics of the
system. Given a smooth curve $\gamma:[0,1]\to\mathbb R^d$, its
\emph{action} is
\[
S[\gamma] \;=\; \int_0^1
\mathcal L\bigl(\gamma(t),\dot\gamma(t),t\bigr)\,dt.
\]
The \emph{least-action principle} selects, between two fixed
endpoints $\gamma(0)=x_0$ and $\gamma(1)=x_1$, the curve that
minimizes $S[\gamma]$. A standard variational argument shows that
any minimizer satisfies the \emph{Euler--Lagrange equation}
\[
\frac{d}{dt}\!\left(\frac{\partial \mathcal L}{\partial v}\right)
- \frac{\partial \mathcal L}{\partial x} \;=\; 0,
\]
which, together with the boundary conditions, defines a boundary-value
problem that admits a unique \emph{closed-form} solution for some choices of Lagrangians, as illustrated by the following two examples.

\begin{example}[Free particle]
\label{ex:free_particle}
The simplest Lagrangian is the kinetic energy of a free particle,
$\mathcal L_{\mathrm{free}}(x,v) = \tfrac12\|v\|^2$. The Euler--Lagrange
equation reduces to $\ddot\gamma(t)=0$, so the least-action trajectory
between $x_0$ and $x_1$ is the affine interpolation
$\gamma_{x_0,x_1}(t) = (1-t)x_0+t x_1$, with constant velocity
$\dot\gamma_{x_0,x_1}(t)=x_1-x_0$. (\Cref{fig:harmonic_pos_vel}, $\omega=0$ curve). The corresponding endpoint cost is
the quadratic cost
$c_{\mathcal L_{\mathrm{free}}}(x_0,x_1)=\tfrac12\|x_1-x_0\|^2$,
underlying classical Wasserstein optimal transport.
$\hfill\blacksquare$
\end{example}
\begin{example}[Harmonic oscillator]
\label{ex:harmonic_oscillator}
A canonical non-kinetic example is
$\mathcal L_\omega(x,v) = \tfrac12\|v\|^2 - \tfrac12\omega^2\|x\|^2$
with a frequency $0<\omega<\pi$, adding a quadratic potential
$V(x)=\tfrac12\omega^2\|x\|^2$ to the kinetic term. The
Euler--Lagrange equation $\ddot\gamma+\omega^2\gamma=0$ has unique
solution
\begin{align}
\label{eq:harmonic_gamma}
\gamma_{x_0,x_1}^{\omega}(t)
\;=\; \frac{\sin(\omega(1-t))}{\sin\omega}\,x_0
+ \frac{\sin(\omega t)}{\sin\omega}\,x_1,
\end{align}
with time-varying velocity
\begin{align}
\label{eq:harmonic_gamma_v}
\dot\gamma_{x_0,x_1}^{\omega}(t)
= -\frac{\omega\cos(\omega(1-t))}{\sin\omega}\,x_0
+ \frac{\omega\cos(\omega t)}{\sin\omega}\,x_1.
\end{align}
As $\omega\to 0$, $\sin(\omega s)/\sin\omega\to s$, so both expressions reduce to those of \Cref{ex:free_particle}. The harmonic oscillator thus smoothly deforms the free-particle dynamics, with curvature controlled by the frequency $\omega$ (\Cref{fig:harmonic_pos_vel}, $\omega = 0, \pi/4, \pi/2$ and $3\pi/4$ curve). $\hfill\blacksquare$
\end{example}
The least-action principle gives a recipe for selecting a deterministic
trajectory between two fixed endpoints. To use this recipe in flow
matching, where endpoints are samples from $p_{\mathrm{init}}$ and
$p_{\mathrm{data}}$ rather than fixed points, we need a way to (i)
couple endpoints across the two distributions, and (ii) assemble the
resulting least-action trajectories into a probability path. This is
the role of \emph{Lagrangian optimal transport}, developed next, which
generalizes the Benamou--Brenier dynamic formulation
\cite{benamou2000computational} from the kinetic Lagrangian to general
Lagrangians \cite{bernard2007optimal,villani2009optimal}.

\subsection{Lagrangian Optimal Transport}
\label{sec:lagrangian_ot}

For $x_0,x_1\in\mathbb R^d$, consider the least-action problem
\begin{equation}
\label{eq:path_min_action}
\gamma_{x_0,x_1}
\;:=\;
\arg\min_{\substack{\gamma\in AC([0,1];\mathbb R^d)\\
\gamma(0)=x_0,\,\gamma(1)=x_1}}
\int_0^1
\mathcal L(\gamma(t),\dot\gamma(t),t)\,dt,
\end{equation}
where $AC([0,1];\mathbb R^d)$ denotes the space of absolutely continuous
curves. 
We restrict throughout to Lagrangians of the classical separable form
\(\mathcal L(x,v,t)=K(v)-V(x)\), with \(K\) a quadratic
form satisfying \(K(v)\ge \alpha\|v\|^2\) for some \(\alpha>0\). The potential
\(V\ge0\) is taken sufficiently regular and controlled
by the kinetic term so that \eqref{eq:path_min_action} admits a
unique minimizer for every \((x_0,x_1)\).
The corresponding least-action cost is
\begin{equation}
\label{eq:cost_L}
c_{\mathcal L}(x_0,x_1)
\;:=\;
\int_0^1
\mathcal L\bigl(
\gamma_{x_0,x_1}(t),
\dot\gamma_{x_0,x_1}(t),
t
\bigr)\,dt.
\end{equation}

Given probability densities $p_{\mathrm{init}}$ and $p_{\mathrm{data}}$
on $\mathbb R^d$, using $c_{\mathcal L}$ as the transport cost yields
the \emph{static} Lagrangian optimal transport problem
\begin{equation}
\label{eq:static_problem}
\mathcal A_{\mathrm{stat}}
\;:=\;
\inf_{\pi\in\Pi(p_{\mathrm{init}},p_{\mathrm{data}})}
\int_{\mathbb R^d\times\mathbb R^d}
c_{\mathcal L}(x_0,x_1)\,d\pi(x_0,x_1),
\end{equation}
where $\Pi(p_{\mathrm{init}},p_{\mathrm{data}})$ is the set of couplings
between $p_{\mathrm{init}}$ and $p_{\mathrm{data}}$. The corresponding
\emph{dynamic} formulation minimizes the action of the density flow:
\begin{equation}
\label{eq:min_act_prob}
\mathcal A_{\mathrm{dyn}}
\;:=\;
\inf_{(p,v)\in\mathfrak D}
\int_0^1\!\!\int_{\mathbb R^d}
\mathcal L(x,v(x,t),t)\,p(x,t)\,dx\,dt,
\end{equation}
where
$\mathfrak D := \bigl\{(p,v):\,
\partial_t p+\nabla\cdot(vp)=0,\,
p(\cdot,0)=p_{\mathrm{init}},\,
p(\cdot,1)=p_{\mathrm{data}}\bigr\}$.

The two formulations are equivalent. Recall that for a measurable map
$T:\mathbb R^m\to\mathbb R^n$ and a probability measure $\mu$ on
$\mathbb R^m$, the pushforward $T_\#\mu$ is defined by
$(T_\#\mu)(A):=\mu(T^{-1}(A))$, and that the evaluation map
$e_t:AC([0,1];\mathbb R^d)\to\mathbb R^d$ is $e_t(\gamma)=\gamma(t)$.

\begin{theorem}[Static--Dynamic Equivalence]
\label{thm:connection_pi_v_p}
The static and dynamic Lagrangian optimal transport problems have the
same value: $\mathcal A_{\mathrm{stat}}=\mathcal A_{\mathrm{dyn}}$.
Moreover, if $\pi^*$ is an optimal plan for \eqref{eq:static_problem},
the probability measure
$\eta^*:=(\gamma_{x_0,x_1})_\#\pi^*$ on $AC([0,1];\mathbb R^d)$
induces an optimal pair for \eqref{eq:min_act_prob} via
\[
p^*(\cdot,t):=(e_t)_\#\eta^*,
\qquad
v^*(x,t)
\;:=\;
\mathbb E_{\gamma\sim\eta^*}\bigl[\dot\gamma(t)\,\big|\,\gamma(t)=x\bigr].
\]
\end{theorem}
\Cref{thm:connection_pi_v_p} shows that the least-action probability path
is obtained by pushing an optimal static coupling forward along
least-action trajectories. This converts training into a two-step
procedure. First, solve the static problem \eqref{eq:static_problem}
for an optimal coupling $\pi^*$. Then, for samples
$(x_0,x_1)\sim\pi^*$, use the closed-form least-action trajectory
$\gamma_{x_0,x_1}$ as a velocity target:
\begin{align}
\label{eq:training_obj}
\min_\theta\;
\mathbb E_{t\sim \mathcal U[0,1]}\,
\mathbb E_{(x_0,x_1)\sim\pi^*}
\bigl\|
v_\theta\bigl(\gamma_{x_0,x_1}(t),t\bigr)
-\dot\gamma_{x_0,x_1}(t)
\bigr\|^2.
\end{align}
This sidesteps the dynamic problem \eqref{eq:min_act_prob} entirely,
preserving the simulation-free character of flow matching. The
following examples illustrate the procedure for the harmonic
Lagrangian of \Cref{ex:harmonic_oscillator} and recover several
existing flow matching constructions as special cases.

\begin{figure}[t]
\centering
\begin{subfigure}[t]{0.195\linewidth}
  \centering
  \includegraphics[width=\linewidth]{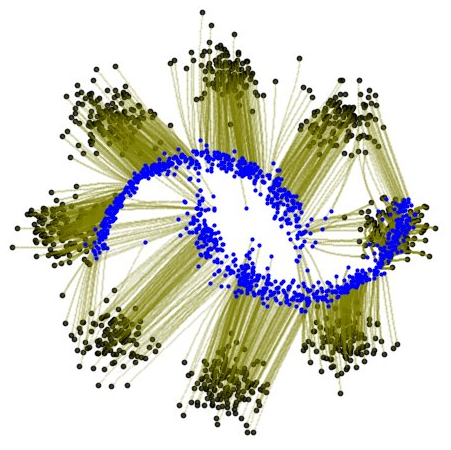}
  \caption{\centering OT-CFM\\($\omega\!\to\!0$)}
  \label{fig:5gaussian_ot}
\end{subfigure}\hfill
\begin{subfigure}[t]{0.195\linewidth}
  \centering
  \includegraphics[width=\linewidth]{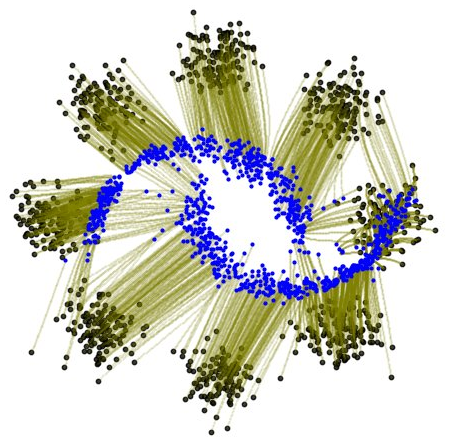}
  \caption{\centering Harmonic\\$\omega=0.001$}
  \label{fig:5gaussian_w001}
\end{subfigure}\hfill
\begin{subfigure}[t]{0.195\linewidth}
  \centering
  \includegraphics[width=\linewidth]{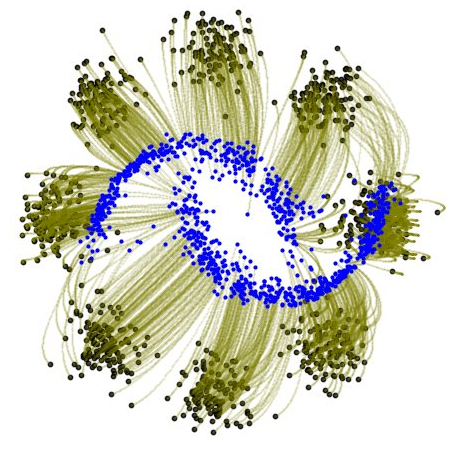}
  \caption{\centering Harmonic\\$\omega=1$}
  \label{fig:5gaussian_w1}
\end{subfigure}\hfill
\begin{subfigure}[t]{0.195\linewidth}
  \centering
  \includegraphics[width=\linewidth]{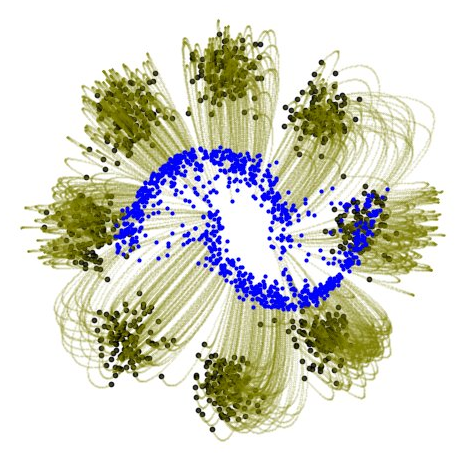}
  \caption{\centering Harmonic\\$\omega=\pi/2$}
  \label{fig:5gaussian_wpi2}
\end{subfigure}\hfill
\begin{subfigure}[t]{0.195\linewidth}
  \centering
  \includegraphics[width=\linewidth]{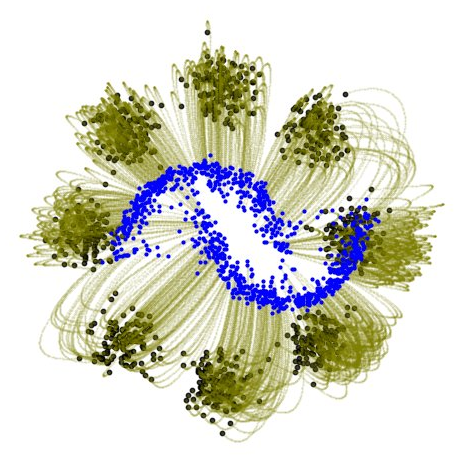}
  \caption{\centering Stochastic\\interpolant}
  \label{fig:5gaussian_si}
\end{subfigure}
\caption{Mini-batch harmonic flow matching trajectories from a 8-Gaussian-mixture to a double-moon target, swept over $\omega\in(0,\pi/2]$ and compared against OT-CFM and the stochastic interpolant. Endpoints (black) are matched by quadratic-cost OT and connected along least-action trajectories (olive); the data manifold is shown in blue. See \Cref{ex:harmonic_to_ot,ex:harmonic_to_trig} for the limiting cases.}
\label{fig:5gaussian_sweep}
\end{figure}

\begin{example}[Harmonic flow]
\label{ex:harmonic_flow}
For $\mathcal L_\omega$ of \Cref{ex:harmonic_oscillator},
substituting \eqref{eq:harmonic_gamma} into the action yields the cost
\[
c_\omega(x_0,x_1)
\;=\; \frac{\omega}{2\sin\omega}
\bigl[\cos\omega\,(\|x_0\|^2+\|x_1\|^2) - 2\,x_0\!\cdot\!x_1\bigr].
\]
For any coupling $\pi\in\Pi(p_{\mathrm{init}},p_{\mathrm{data}})$, the
quadratic terms in $c_\omega$ integrate to constants determined by the
marginals alone:
\begin{equation*}
\int \|x_0\|^2\,d\pi(x_0,x_1)
= \mathbb E_{x_0\sim p_\mr{init}}\|x_0\|^2
=: M_0,\quad
\int \|x_1\|^2\,d\pi(x_0,x_1)
= \mathbb E_{x_1\sim p_\mr{data}}\|x_1\|^2
=: M_1,
\end{equation*}
which are finite under the standard second-moment assumption on
$p_{\mathrm{init}}$ and $p_{\mathrm{data}}$. The transport cost
therefore decomposes as
\[
\int c_\omega(x_0,x_1)\,d\pi(x_0,x_1)
\;=\; \frac{\omega\cos\omega}{2\sin\omega}\,(M_0+M_1)
\;-\; \frac{\omega}{\sin\omega}\int x_0\!\cdot\!x_1\,d\pi(x_0,x_1),
\]
where the first term does not depend on $\pi$. Minimizing over
$\pi$ is therefore equivalent to maximizing the cross-correlation
$\int x_0\!\cdot\!x_1\,d\pi$, and by the polarization identity
$x_0\!\cdot\!x_1 = \tfrac12(\|x_0\|^2+\|x_1\|^2-\|x_0-x_1\|^2)$, this
is in turn equivalent to minimizing
$\int \tfrac12\|x_0-x_1\|^2\,d\pi$. The static problem
\eqref{eq:static_problem} with cost $c_\omega$ thus shares its
minimizer with the quadratic-cost OT problem for every
$\omega\in(0,\pi)$; we denote this common coupling by
$\pi^*_{\mathrm{OT}}$. The training objective
\eqref{eq:training_obj} becomes
\[
\min_\theta\;
\mathbb E_{t,\,(x_0,x_1)\sim\pi^*_{\mathrm{OT}}}
\bigl\|
v_\theta\bigl(\gamma_{x_0,x_1}^{\omega}(t),t\bigr)
-\dot\gamma_{x_0,x_1}^{\omega}(t)
\bigr\|^2,
\]
where the position $\gamma_{x_0,x_1}^{\omega}(t)$ and velocity
$\dot\gamma_{x_0,x_1}^{\omega}(t)$ of the harmonic oscillator are given
by \Cref{eq:harmonic_gamma,eq:harmonic_gamma_v}. This preserves the
OT endpoint coupling while replacing the affine interpolation between
paired endpoints with the curved harmonic trajectory.
\Cref{fig:5gaussian_sweep} visualizes how the resulting trajectories
vary with $\omega \in \{0.001,\, 1,\, \pi/2\}$ on an 8-Gaussian-mixture
source and a two-moons target.
$\hfill\blacksquare$
\end{example}

\begin{example}[OT-based flow matching as the $\omega\to 0$ limit]
\label{ex:harmonic_to_ot}
Taking $\omega\to 0$ in \Cref{ex:harmonic_flow}, the harmonic
trajectory reduces to the affine interpolation $(1-t)x_0+t x_1$ and
the cost reduces to $\tfrac12\|x_1-x_0\|^2$, so the training objective
becomes
\[
\min_\theta\;
\mathbb E_{t,\,(x_0,x_1)\sim\pi^*_{\mathrm{OT}}}
\bigl\|
v_\theta\bigl((1-t)x_0+t x_1,\,t\bigr)
-(x_1-x_0)
\bigr\|^2,
\]
recovering OT-based flow matching \cite{tong2024improving}
(\Cref{fig:5gaussian_ot}; the small-$\omega$ case in \Cref{fig:5gaussian_w001} is visually indistinguishable). The
affine construction of \cite{lipman2023flow} is the conditional
analogue of this limit, recovered in \Cref{sec:conditional_lfm} by
a different choice of conditioning variable. $\hfill\blacksquare$
\end{example}

\begin{example}[Trigonometric interpolation at $\omega=\pi/2$]
\label{ex:harmonic_to_trig}
At $\omega=\pi/2$, the harmonic trajectory reduces to the
trigonometric interpolation
\[
\gamma_{x_0,x_1}^{\pi/2}(t)
\;=\; \cos\!\left(\frac{\pi t}{2}\right)x_0
+ \sin\!\left(\frac{\pi t}{2}\right)x_1,
\]
which connects to two existing constructions. First, it is the
spatial interpolation underlying stochastic-interpolant
constructions \cite{albergo2025stochastic}
(\Cref{fig:5gaussian_sweep}\subref{fig:5gaussian_wpi2} and
\subref{fig:5gaussian_si} show the two constructions side by side);
here, however, the trajectory is fully deterministic once endpoints
are fixed, with randomness entering only through endpoint sampling.
Second, with $p_{\mathrm{init}}=\mathcal N(0,I)$, the trajectory reads
$x_t = \sigma_t x_0 + \alpha_t x_1$ where $\sigma_t=\cos(\pi t/2)$ and $\alpha_t=\sin(\pi t/2)$ satisfy the variance-preserving normalization $\sigma_t^2+\alpha_t^2=1$ used in diffusion-type paths \cite{lipman2023flow,ho2020denoising}. The harmonic Lagrangian thus provides a least-action interpretation of the trigonometric VP schedule at the level of marginal probability path; the full diffusion model also specifies a stochastic forward process, which lies outside the current deterministic framework. $\hfill\blacksquare$
\end{example}

\section{Conditional Lagrangian Flow Matching}
\label{sec:conditional_lfm}

The least-action problem \eqref{eq:min_act_prob} provides a principled
mechanism for selecting probability paths and velocity fields, but
solving it directly can be challenging in practice. While the
static--dynamic equivalence (\Cref{thm:connection_pi_v_p}) reduces the
problem to a static optimal transport problem between
$p_{\mathrm{init}}$ and $p_{\mathrm{data}}$, the resulting OT problem
may still be difficult to solve accurately, particularly in
high-dimensional or sample-only regimes. A common strategy in flow
matching is to represent a global probability path as a mixture of
simpler conditional paths
\cite{lipman2023flow,tong2024improving}. We adapt this strategy to
Lagrangian Flow Matching: each conditional path is selected by a
least-action principle, and the Lagrangian is allowed to depend on the
conditioning variable.

Let $z\sim q$ be a conditioning variable, and let $p_0(\cdot\mid z)$
and $p_1(\cdot\mid z)$ be conditional endpoint distributions whose
mixtures recover the desired marginals:
\begin{equation}
\label{eq:conditional_endpoint_mixtures}
\int p_0(\cdot\mid z)\,q(dz) = p_{\mathrm{init}},
\qquad
\int p_1(\cdot\mid z)\,q(dz) = p_{\mathrm{data}}.
\end{equation}
The conditional endpoints need not match $p_{\mathrm{init}}$ and
$p_{\mathrm{data}}$ individually; only their mixtures must. For each
$z$, let $\mathcal L_z(x,v,t)$ be a (possibly $z$-dependent) Lagrangian and let
$
S_z[\gamma] \;=\; \int_0^1
\mathcal L_z\bigl(\gamma(t),\dot\gamma(t),t\bigr)\,dt
$ be the corresponding action. Then, the associated $z$-conditional least-action trajectory and cost are
\begin{equation}
\label{eq:conditional_lagrangian_cost}
\begin{aligned}
\gamma_{x_0,x_1}^z
\;:=\;
\arg\min_{\substack{\gamma\in AC([0,1];\mathbb R^d)\\
\gamma(0)=x_0,\,\gamma(1)=x_1}}
S_z[\gamma],\qquad
c_z(x_0,x_1)
\;:=\;
S_z[\gamma_{x_0,x_1}^z].
\end{aligned}
\end{equation}
For each $z$, the conditional endpoint coupling is obtained by solving
the $z$-conditional static optimal transport problem
\begin{equation}
\label{eq:conditional_static_ot}
\pi_z^*
\;\in\;
\operatorname*{argmin}_{\pi\in\Pi(p_0(\cdot\mid z),p_1(\cdot\mid z))}
\int_{\mathbb R^d\times\mathbb R^d}
c_z(x_0,x_1)\,d\pi(x_0,x_1).
\end{equation}
Pushing $\pi_z^*$ forward along the least-action trajectories yields
the conditional probability path and velocity field
\begin{equation}
\label{eq:conditional_path_velocity}
\!\!\!\!\!\!\!\!\!\! p_t(\cdot\mid z)
\;=\; 
\bigl(\gamma_{x_0,x_1}^z(t)\bigr)_{\#}
\pi_z^*,
\;\;\;
v_t(x\mid z)
\;=\; \mathbb E_{(x_0,x_1)\sim\pi_z^*}
\bigl[\dot\gamma_{x_0,x_1}^z(t)\,\big|\,\gamma_{x_0,x_1}^z(t)=x\bigr].
\end{equation}
Applying \Cref{thm:connection_pi_v_p} to each $z$ independently,
$(p_t(\cdot\mid z), v_t(\cdot\mid z))$ is the optimal solution of the
$z$-conditional dynamic least-action problem
\begin{align}
\label{eq:min_act_prob_cond}
\mathcal A_{\mathrm{cond}}
\;:=\;
\inf_{\{p_t(\cdot\mid z),v_t(\cdot\mid z)\}_z}\;
\mathbb E_{z\sim q}
\int_0^1\!\!\int_{\mathbb R^d}
\mathcal L_z\bigl(x,v_t(x\mid z),t\bigr)\,p_t(x\mid z)\,dx\,dt,
\end{align}
with $p_t(\cdot|z)$ and $v_t(\cdot|z)$ subject to the continuity equation and endpoint constraints.

The least-action trajectories $\gamma_{x_0,x_1}^z$ provide the following
training targets for the conditional Lagrangian flow matching
objective
\begin{equation}
\label{eq:conditional_lfm_objective}
\min_\theta\;
\mathbb E_{z\sim q}\,
\mathbb E_{(x_0,x_1)\sim \pi_z^*}\,
\mathbb E_{t\sim \mathcal U[0,1]}
\bigl\|
v_\theta(\gamma_{x_0,x_1}^z(t),t)
-\dot\gamma_{x_0,x_1}^z(t)
\bigr\|^2,
\end{equation}
which generalizes the unconditional objective \eqref{eq:training_obj}.
The conditioning distribution $q$, conditional endpoints $p_0(\cdot|z)$ and $p_1(\cdot|z)$, and
$z$-dependent Lagrangians $\mathcal L_z$ are design parameters that
can be chosen for the problem at hand, while the simulation-free
property is preserved provided the trajectories $\gamma_{x_0,x_1}^z$
admit closed-form expressions. The next proposition shows that the
conditional formulation is consistent with the unconditional one when
the same convex Lagrangian is used.

\begin{proposition}[Conditional action upper-bounds the unconditional action]
\label{prop:conditional_action_dominates}
Assume $\mathcal L_z=\mathcal L$ for all $z$ and that
$\mathcal L(x,\cdot,t)$ is convex in its velocity argument. Then
$\mathcal A_{\mathrm{cond}}\ge \mathcal A_{\mathrm{dyn}}$, where
$\mathcal A_{\mathrm{dyn}}$ is the unconditional least-action value
of \eqref{eq:min_act_prob}. 
Moreover, 
equality is attained by the trivial conditioning
$q=\delta_{z_0}$ with $p_0(\cdot\mid z_0)=p_{\mathrm{init}}$ and
$p_1(\cdot\mid z_0)=p_{\mathrm{data}}$.
\end{proposition}
The conditional formulation thus never undershoots the unconditional optimum and matches it in the trivial case. More details are deferred to \Cref{app:_d2}.

\begin{figure}[t]
\centering
\begin{subfigure}[t]{0.195\linewidth}
  \centering
  \includegraphics[width=\linewidth]{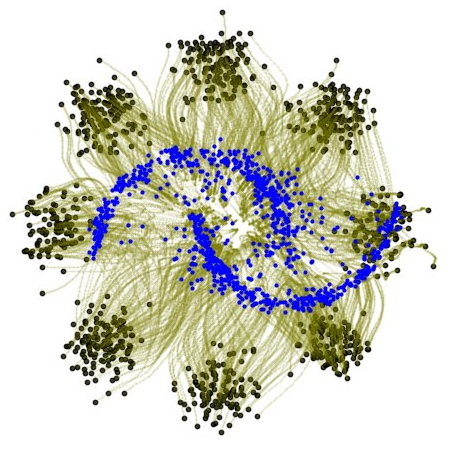}
  \caption{$n=1$}
  \label{fig:ot_batch_sweep_1}
\end{subfigure}\hfill
\begin{subfigure}[t]{0.195\linewidth}
  \centering
  \includegraphics[width=\linewidth]{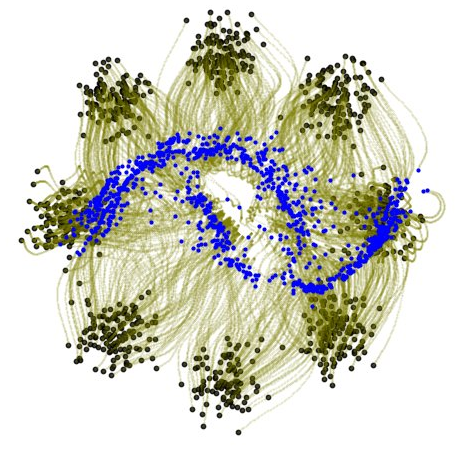}
  \caption{$n=10$}
  \label{fig:ot_batch_sweep_10}
\end{subfigure}\hfill
\begin{subfigure}[t]{0.195\linewidth}
  \centering
  \includegraphics[width=\linewidth]{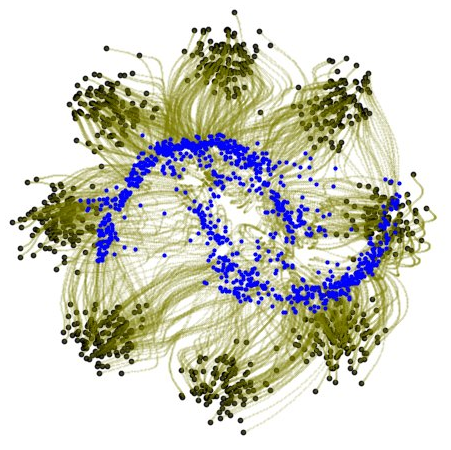}
  \caption{$n=25$}
  \label{fig:ot_batch_sweep_25}
\end{subfigure}\hfill
\begin{subfigure}[t]{0.195\linewidth}
  \centering
  \includegraphics[width=\linewidth]{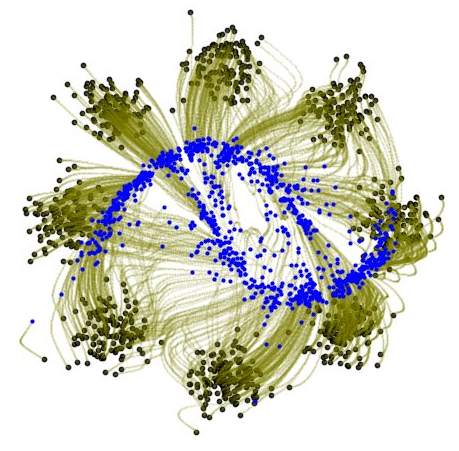}
  \caption{$n = 50$}
  \label{fig:ot_batch_sweep_50}
\end{subfigure}\hfill
\begin{subfigure}[t]{0.195\linewidth}
  \centering
  \includegraphics[width=\linewidth]{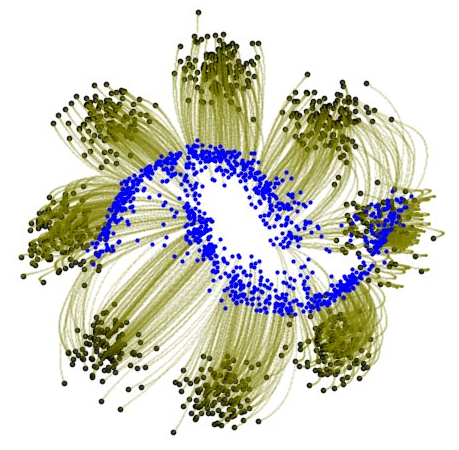}
  \caption{$n=100$}
  \label{fig:ot_batch_sweep_100}
\end{subfigure}
\caption{Mini-batch harmonic flow matching ($\omega = 1$) from an 8-Gaussian-mixture source (olive) to a double-moon target (blue), swept over OT batch size $n \in \{1, 10, 25, 50, 100\}$. Increasing $n$ tightens the empirical coupling toward the OT solution: compare the many-to-many fan-out at $n=1$ to the bundled, mode-to-region routing at $n=100$. Trajectory curvature is set by $\omega$ and is preserved across all panels --- the two axes ($n$ for the coupling, $\mathcal{L}$ for the matched-pair trajectory) act independently.}
\label{fig:ot_batch_sweep}
\end{figure}

\subsection{Mini-batch Conditional Lagrangian Flow Matching}
\label{sec:minibatch_lfm}

A natural and useful instance of the conditional formulation
is \emph{mini-batch} conditional Lagrangian flow matching, in which
the conditioning variable is a pair of finite batches drawn
independently from $p_{\mathrm{init}}$ and $p_{\mathrm{data}}$. This
generalizes mini-batch optimal-transport flow matching
\cite{tong2024improving,pooladian2023multisample} from the kinetic
Lagrangian to general Lagrangians, and is well-suited to the
sample-only regime in which $p_{\mathrm{init}}$ and $p_{\mathrm{data}}$
are accessed only through samples.

Fix a batch size $n\ge 1$, and take the conditioning variable to be a
pair of independent batches
$z = (x_0^{(1:n)},\,x_1^{(1:n)})$ drawn from
$p_{\mathrm{init}}^{\otimes n}\otimes p_{\mathrm{data}}^{\otimes n}$,
with conditional endpoint distributions given by the corresponding
empirical distributions
$p_0(\cdot\mid z) = \tfrac1n\sum_{i=1}^n \delta_{x_0^{(i)}}$ and
$p_1(\cdot\mid z) = \tfrac1n\sum_{i=1}^n \delta_{x_1^{(i)}}$. The
mixture conditions \eqref{eq:conditional_endpoint_mixtures} hold by
construction. We use a single, $z$-independent Lagrangian
$\mathcal L$ of the form $K(v)-V(x)$, with associated trajectory
$\gamma_{x_0,x_1}$ and cost $c_{\mathcal L}(x_0,x_1)$ as in
\eqref{eq:cost_L}. With these choices, the conditional static
problem \eqref{eq:conditional_static_ot} reduces to a discrete
optimal transport problem between the two empirical batches:
\begin{equation}
\label{eq:minibatch_discrete_ot}
\pi_z^*
\;\in\;
\operatorname*{argmin}_{\pi\in\Pi(p_0(\cdot\mid z),p_1(\cdot\mid z))}
\sum_{i,j}\pi_{ij}\,c_{\mathcal L}(x_0^{(i)},x_1^{(j)}),
\end{equation}
which is a small linear assignment problem solvable by standard
methods (e.g., the Hungarian algorithm or Sinkhorn iterations) at
every gradient step. Substituting the matched pairs into \eqref{eq:conditional_lfm_objective}
and writing the conditional expectation explicitly over batch indices
gives the mini-batch conditional Lagrangian flow matching objective
\begin{equation}
\label{eq:minibatch_lfm_objective}
\min_\theta\;
\mathbb E_{z\sim p_{\mathrm{init}}^{\otimes n}\otimes p_{\mathrm{data}}^{\otimes n}}\,
\mathbb E_{(i,j)\sim \pi_z^*}\,
\mathbb E_{t\sim \mathcal U[0,1]}
\bigl\|
v_\theta\bigl(\gamma_{x_0^{(i)},x_1^{(j)}}(t),t\bigr)
-\dot\gamma_{x_0^{(i)},x_1^{(j)}}(t)
\bigr\|^2,
\end{equation}
where $\pi_z^*$ from \eqref{eq:minibatch_discrete_ot} is viewed as a
distribution over index pairs $(i,j)\in\{1,\dots,n\}^2$. Each gradient step samples a batch $z$, solves \eqref{eq:minibatch_discrete_ot}, and uses $\gamma_{x_0,x_1}$ as the velocity target along matched pairs. As $n\to\infty$, the empirical marginals converge to $p_{\mathrm{init}}$ and $p_{\mathrm{data}}$, and $\pi_z^*$ converges to the global optimal coupling $\pi^*$ under mild regularity. For finite $n$ the coupling is approximate, but the objective remains simulation-free for any Lagrangian admitting closed-form least-action trajectories.\footnote{
In practice, the OT batch size and the training batch size need not be the same.
For instance, when the OT batch size \(n\) is too small, one may aggregate the outputs of several OT batches to form a larger training batch.}

\Cref{fig:ot_batch_sweep} illustrates how the batch size \(n\)
affects the trajectories in mini-batch harmonic flow matching. As \(n\)
increases, the endpoint pairings gradually transition from independent
pairings to OT pairing. Consequently,
the resulting training trajectories increasingly reflect the global OT geometry. In the limiting case \(\omega\to 0\), it recovers standard OT flow matching as
\(n\to\infty\), and recovers rectified or affine conditional flow matching
when \(n=1\). Cases with \(\omega\neq 0\) can therefore be viewed as
natural extensions of these existing frameworks. See \Cref{app:harmonic_flow_matching} for a detailed discussion with examples.

\begin{table}[t]
\centering
\caption{Comparison of harmonic flow matching across four distribution
pairs ($\mu \pm \sigma$ over five seeds) in terms of distribution fit
(2-Wasserstein), deviation from the harmonic OT path at
$\omega = 1$ (normalized path energy), and training time. Lower is
better for both metrics; best is in \textbf{bold}.}
\label{tab:harmonic_comparison}
\setlength{\tabcolsep}{4pt}
\renewcommand{\arraystretch}{1.15}
\resizebox{\textwidth}{!}{%
\begin{tabular}{l cc cc cc cc c}
\toprule
Dataset $\rightarrow$
 & \multicolumn{2}{c}{$\mathcal{N}\!\rightarrow\!\text{8gaussians}$}
 & \multicolumn{2}{c}{moons$\leftrightarrow$8gaussians}
 & \multicolumn{2}{c}{$\mathcal{N}\!\rightarrow\!\text{moons}$}
 & \multicolumn{2}{c}{$\mathcal{N}\!\rightarrow\!\text{scurve}$}
 & Avg.\ train time \\
\cmidrule(lr){2-3}\cmidrule(lr){4-5}\cmidrule(lr){6-7}\cmidrule(lr){8-9}\cmidrule(lr){10-10}
Algorithm $\downarrow$ Metric $\rightarrow$
 & $W_2$ & $\text{NPE}_{\omega=1}$
 & $W_2$ & $\text{NPE}_{\omega=1}$
 & $W_2$ & $\text{NPE}_{\omega=1}$
 & $W_2$ & $\text{NPE}_{\omega=1}$
 & ($\times 10^3$ s) \\
\midrule
OT-CFM
 & $0.505_{\pm 0.118}$ & $0.087_{\pm 0.037}$
 & $0.336_{\pm 0.113}$ & $0.187_{\pm 0.056}$
 & $0.197_{\pm 0.026}$ & $0.138_{\pm 0.047}$
 & $0.976_{\pm 0.198}$ & $0.051_{\pm 0.041}$
 & $\mathbf{0.827}_{\pm 0.037}$ \\
OT-SI
 & $0.521_{\pm 0.092}$ & $0.497_{\pm 0.042}$
 & $0.401_{\pm 0.154}$ & $1.308_{\pm 0.148}$
 & $0.227_{\pm 0.036}$ & $0.785_{\pm 0.117}$
 & $1.126_{\pm 0.183}$ & $0.325_{\pm 0.041}$
 & $0.888_{\pm 0.095}$ \\
\midrule
OT-Harmonic, $\omega=0.001$
 & $0.504_{\pm 0.119}$ & $0.087_{\pm 0.037}$
 & $\mathbf{0.335}_{\pm 0.113}$ & $0.186_{\pm 0.058}$
 & $0.197_{\pm 0.024}$ & $0.136_{\pm 0.047}$
 & $\mathbf{0.975}_{\pm 0.199}$ & $0.051_{\pm 0.041}$
 & $0.898_{\pm 0.080}$ \\
 OT-Harmonic, $\omega=1$
 & $\mathbf{0.493}_{\pm 0.117}$ & $\mathbf{0.025}_{\pm 0.030}$
 & $0.3725_{\pm 0.181}$ & $\mathbf{0.046}_{\pm 0.035}$
 & $\mathbf{0.188}_{\pm 0.016}$ & $\mathbf{0.036}_{\pm 0.029}$
 & $1.023_{\pm 0.241}$ & $\mathbf{0.037}_{\pm 0.008}$
 & $0.845_{\pm 0.038}$ \\
 OT-Harmonic, $\omega=\pi/2$
 & $0.520_{\pm 0.092}$ & $0.497_{\pm 0.045}$
 & $0.402_{\pm 0.155}$ & $1.308_{\pm 0.149}$
 & $0.228_{\pm 0.037}$ & $0.785_{\pm 0.118}$
 & $1.127_{\pm 0.186}$ & $0.325_{\pm 0.041}$
 & $0.907_{\pm 0.115}$ \\
\bottomrule
\end{tabular}%
}
\end{table}

\section{Experiments}
\label{sec:experiments}
We evaluate Lagrangian flow matching on three settings of increasing
dimensionality: synthetic two-dimensional data, single-cell
trajectory interpolation, and CIFAR-10 image generation. We
instantiate three OT-Harmonic models at frequencies
$\omega \in \{0.001,\, 1,\, \pi/2\}$ and compare against OT-CFM
\citep{tong2024improving} and OT-SI \citep{albergo2023building},
with an additional anisotropic variant (OT-Aniso, see \Cref{app:anisotropic_harmonic} for details) on CIFAR-10. Baselines are retrained under our own procedure so that all methods share the same architecture, optimizer, training budget, and evaluation protocol for controlled comparison.

Three findings emerge. First, the OT-Harmonic family interpolates between
the two baselines along $\omega$: as $\omega \to 0$, harmonic
geodesics degenerate to straight lines and recovers OT-CFM, while at $\omega = \pi/2$ the interpolant reduces to the cosine schedule underlying OT-SI. Second, $\omega = 0.001$ slightly improves on OT-CFM on Multiome and on CIFAR-10, suggesting that minor deviations from straight-line interpolation can be beneficial even in the limit where the harmonic flow approaches OT-CFM theoretically. Third, $\omega = 1$ offers the largest gains in the 2D setting, improving on both baselines in $W_2$ on three of four distribution pairs; in higher-dimensional settings it remains competitive but no longer dominates. We defer the experimental setup and additional simulation results to \Cref{app:exp_setup}.

\textbf{Two-dimensional data.} To assess how closely a learned flow recovers the harmonic optimal transport solution, we report the
normalized path energy $\mathrm{NPE}_{\omega}[\phi] = \left|\, K[\phi] / C_{\omega}(\pi_0,\pi_1) - 1 \,\right|$, the relative deviation of the flow's kinetic energy $K[\phi]$ from the harmonic-OT reference cost $C_\omega(\pi_0,\pi_1)$ at frequency $\omega$ (closed-form derivation in \Cref{app:npe-derivation}). $\mathrm{NPE}_\omega = 0$ iff $\phi$ is the harmonic-OT flow at $\omega$.

\begin{wrapfigure}[10]{r}{0.7\linewidth}
\vspace{-\baselineskip}
\centering
\begin{subfigure}[t]{0.333\linewidth}
  \centering
  \includegraphics[width=\linewidth]{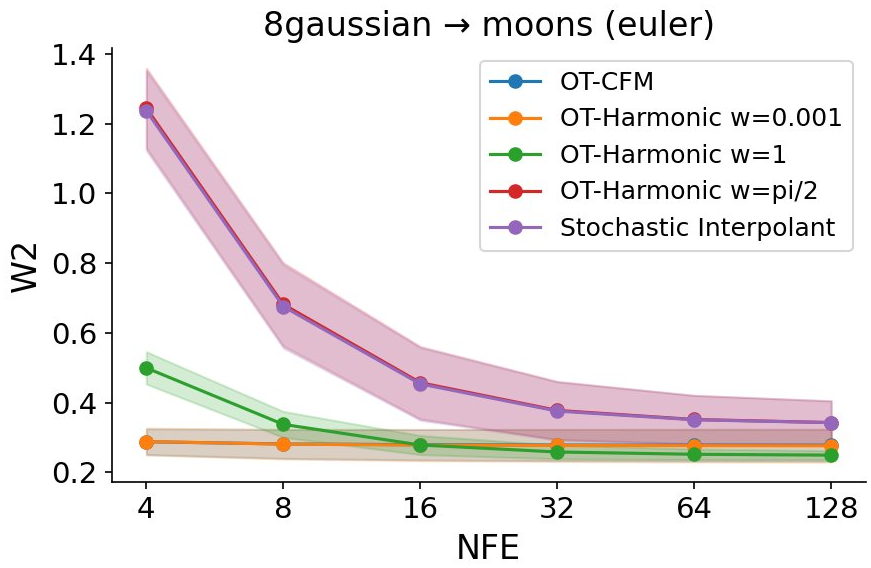}
\end{subfigure}\hfill
\begin{subfigure}[t]{0.333\linewidth}
  \centering
  \includegraphics[width=\linewidth]{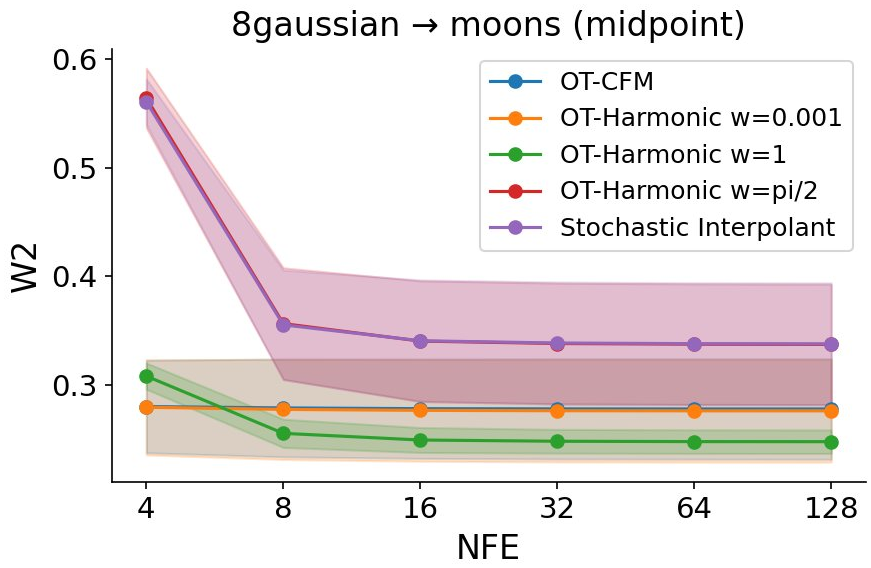}
\end{subfigure}\hfill
\begin{subfigure}[t]{0.333\linewidth}
  \centering
  \includegraphics[width=\linewidth]{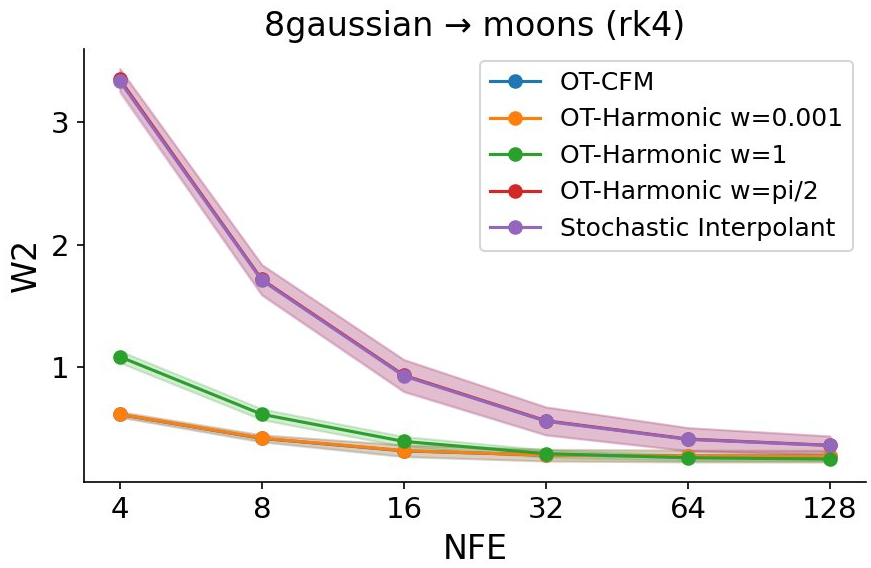}
\end{subfigure}
\caption{Sample quality vs.\ inference budget on
$\mathcal{N}\!\to\!\text{moons}$ across three ODE solvers (Euler,
midpoint, RK4). OT-Harmonic reduce the 2-Wasserstein distance for a fixed NFE during inference.}
\label{fig:solver_small}
\end{wrapfigure}
\Cref{tab:harmonic_comparison} summarizes the results. At $\omega = 1$
the learned flow tracks the harmonic-OT reference path closely
(low $\mathrm{NPE}_{\omega=1}$) while improving on OT-CFM in $W_2$ on
$\mathcal{N}\!\to\!\text{8gaussians}$ and $\mathcal{N}\!\to\!\text{moons}$.
At $\omega = 0.001$ OT-Harmonic reproduces OT-CFM to within seed
variation, confirming the linear-interpolant limit predicted by the
$\omega \to 0$ analysis. We further examine the inference-time sample
quality in \Cref{fig:solver_small} (also \Cref{fig:solver} in \Cref{app:exp_setup}): OT-Harmonic requires more function
evaluations as $\omega$ grows, consistent with $\omega$ controlling the path curvature. 

\begin{wraptable}[10]{r}{0.62\textwidth}
\vspace{-1\baselineskip}
\centering
\caption{Single-cell trajectory interpolation across datasets,
averaged over leave-one-out held-out timepoints. 
Bold indicates methods within the standard deviation of the lowest mean.}
\label{tab:single_cell}
\setlength{\tabcolsep}{4pt}
\renewcommand{\arraystretch}{1.15}
\resizebox{\linewidth}{!}{%
\begin{tabular}{l ccc}
\toprule
Algorithm $\downarrow$ Dataset $\rightarrow$ & Cite & EB & Multi \\
\midrule
OT-CFM
 & $\mathbf{0.8991 \pm 0.0340}$ & $\mathbf{0.9519 \pm 0.0405}$ & $\mathbf{1.0721 \pm 0.0587}$ \\
OT-SI
 & $1.2424 \pm 0.0343$ & $1.3378 \pm 0.1381$ & $1.4856 \pm 0.0698$ \\
\midrule
OT-Harmonic, $\omega = 0.001$
 & $\mathbf{0.8985 \pm 0.0374}$ & $\mathbf{0.9523 \pm 0.0404}$ & $\mathbf{1.0708 \pm 0.0574}$ \\
OT-Harmonic, $\omega = 1$
 & $0.9654 \pm 0.0298$ & $0.9909 \pm 0.0567$ & $1.1562 \pm 0.0765$ \\
OT-Harmonic, $\omega = \pi/2$
 & $1.2373 \pm 0.0226$ & $1.3352 \pm 0.1334$ & $1.4938 \pm 0.0742$ \\
\bottomrule
\end{tabular}
}
\end{wraptable}
\textbf{Single-cell interpolation.} We next evaluate on single-cell
trajectory interpolation, following the leave-one-out protocol of
\citep{tong2024improving}: the model interpolates the cell
distribution at a held-out timepoint $t$ given snapshots at
$[0, t{-}1]$ and $[t{+}1, T]$. \Cref{tab:single_cell} reports the
$1$-Wasserstein distance to the held-out distribution on three
datasets: embryoid body \citep{moon2019visualizing}, CITE-seq, and
Multiome \citep{burkhardt2022multimodal}. OT-CFM and OT-Harmonic at
$\omega = 0.001$ are the two strongest methods and remain within
$1.5 \times 10^{-3}$ on all three datasets, consistent with the
$\omega \to 0$ limit in which the harmonic flow recovers OT-CFM.
Increasing $\omega$ degrades performance monotonically: $\omega = 1$
falls behind both leaders by $4$--$8\%$ on every dataset, and
$\omega = \pi/2$ collapses to the OT-SI regime.

\begin{wraptable}[25]{r}{0.65\textwidth}
\vspace{-1\baselineskip}
  \caption{FID score and number of function evaluations (NFE) for different ODE solvers: fixed-step Euler integration with 100 and 1000 steps and adaptive integration (\textsc{dopri5}). We have run OT-CFM, OT-SI following our training procedure. The four last rows report the results of our proposed methods OT-Harmonic and its anisotropic variant (OT-Aniso).}
  \label{tab:fid_nfe}
  \centering
  \small
    \begin{tabular}{l cc cc}
    \toprule
    NFE / sample $\rightarrow$ & 100   & 1000  & \multicolumn{2}{c}{Adaptive} \\
    \cmidrule(lr){2-2}\cmidrule(lr){3-3}\cmidrule(lr){4-5}
    Algorithm $\downarrow$     & FID   & FID   & FID   & NFE \\
    \midrule
    OT-CFM                          & 4.730 & 3.901 & 3.706 & $\mathbf{133.12}$ \\
    OT-SI                           & $\mathbf{4.474}$ & 4.077 & 4.003 & 163.84 \\
    \midrule
    OT-Harmonic, $\omega = 0.001$   & 4.656 & $\mathbf{3.849}$ & $\mathbf{3.681}$ & $\mathbf{133.12}$ \\
    OT-Harmonic, $\omega = 1$       & 4.741 & 3.963 & 3.890 & 143.36 \\
    OT-Harmonic, $\omega = \pi/2$   & 4.584 & 4.184 & 4.116 & 163.85 \\
    OT-Aniso                        & 4.821 & 4.168 & 4.078 & 153.60 \\
    \bottomrule
  \end{tabular}
\captionsetup{name=Figure}
\centering
\begin{subfigure}[t]{0.5\linewidth}
  \centering
  \includegraphics[width=\linewidth]{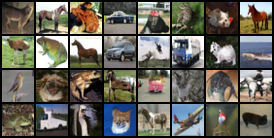}
\end{subfigure}\hfill
\begin{subfigure}[t]{0.5\linewidth}
  \centering
  \includegraphics[width=\linewidth]{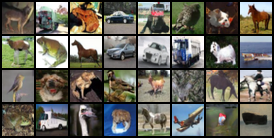}
\end{subfigure}\hfill
\renewcommand{\thetable}{\thefigure}
\refstepcounter{figure}
\addtocounter{table}{-1}
\label{fig:cifar10_small}
\vspace{-0.8\baselineskip}
\caption{CIFAR-10 images generated by OT-CFM (left) and OT-Harmonic (right) with frequency $\omega = 0.001$. }
\end{wraptable}
\textbf{CIFAR-10.}
\Cref{tab:fid_nfe} reports FID and NFE on CIFAR-10, and
\Cref{fig:cifar10_small} shows uncurated samples from OT-CFM and
OT-Harmonic at $\omega = 0.001$; additional sample grids are in
\Cref{fig:cifar10}. Under adaptive integration, FID rises
monotonically with $\omega$ across the harmonic family, recovering
OT-CFM in the small-$\omega$ limit and approaching OT-SI as
$\omega \to \pi/2$. At $\omega = 0.001$, OT-Harmonic reaches
$\mathrm{FID} = 3.681$, slightly improving on our OT-CFM baseline
($\mathrm{FID} = 3.706$) at the same adaptive NFE budget; the two
sample grids in \Cref{fig:cifar10_small} are visually
indistinguishable, consistent with the $\omega \to 0$ correspondence
between the two methods. OT-Harmonic at $\omega = 0.001$ also attains
the lowest FID at $1000$-step Euler integration ($3.849$ vs.\
$3.901$), while at the low-budget $100$-step setting OT-SI attains
the lowest FID ($4.474$); we read this as an inference-time effect
of the cosine schedule rather than a training-time advantage, since
OT-SI trails all harmonic and OT-CFM variants once enough
integration steps are available. Our OT-CFM number is above the
$\mathrm{FID} = 3.577$ reported in \citep{tong2024improving},
reflecting our shorter $400{,}000$-step training budget, which we
apply uniformly across all variants for a controlled comparison. The
anisotropic variant OT-Aniso achieves $\mathrm{FID} = 4.078$, no
better than OT-SI: when CIFAR-10 pixel coordinates are approximately
independent, the rotation recovered from a small fit batch carries
little geometric signal beyond an isotropic schedule.
\section{Conclusion}
\label{sec:conclusion}

We have proposed Lagrangian flow matching, a physics-based framework
that selects flow matching probability paths and velocity fields by
minimizing a Lagrangian action subject to the continuity equation
and prescribed endpoints. The static--dynamic equivalence of
Lagrangian optimal transport reduces this problem to a static OT
problem with a closed-form transport cost, yielding a family of
simulation-free training objectives that recover OT-based and
conditional flow matching as the kinetic special case and admit new
constructions, such as harmonic flows, in the non-kinetic case.
Several directions remain open. The choice of Lagrangian provides
an inductive bias on the learned dynamics whose interaction with
data geometry is not yet well understood; in particular, anisotropic
or learned potentials $V$ may align trajectories with low-dimensional
data structure. Extending the framework to stochastic dynamics, in
the spirit of Schr\"{o}dinger bridges, is a natural next step.

\section*{Acknowledgments}
This work was supported in part by the AWS Build on Trainium program and the Theta EdgeCloud computing resources. The authors gratefully acknowledge these providers for the computing infrastructure that enabled the experiments in this paper.

\bibliographystyle{abbrv}
\bibliography{ref}

\appendix
\crefalias{section}{appendix} 

\section{Related Work}
\label{app:related_work}

Lagrangian flow matching sits at the intersection of two lines of work:
flow- and diffusion-based generative models built on prescribed probability
paths, and variational or action-based approaches to learning transport
dynamics. We discuss connections to representative works in each line.

\subsection{Flow- and Diffusion-based Generative Models}
\label{app:related_work_flow}

\textbf{Continuous Normalizing Flows and Neural ODEs
\citep{chen2018neural,grathwohl2018scalable}.}
Continuous normalizing flows model the data distribution as the pushforward
of a simple base distribution under the flow of a neural ODE
$\dot{x} = v_\theta(x, t)$
\citep{chen2018neural,grathwohl2018scalable}. The induced density evolves
according to the continuity equation, and likelihood-based training
maximizes the change-of-variables log-likelihood of observed data, which
requires solving the ODE forward and the adjoint backward at every gradient
step. This formulation is expressive enough to represent a broad class of
continuous probability paths \citep{song2021maximum}, but the repeated ODE
solves make training computationally expensive and difficult to scale.
Flow matching, including our framework, retains the same deterministic ODE
model at inference time but replaces likelihood-based training with a
regression objective against a prescribed velocity field, which removes
ODE solves from the training loop. Our contribution lies upstream of this
regression step: rather than prescribing the probability path and velocity
field by hand, we select them as the least-action solution of a Lagrangian,
providing a principled mechanism for shaping the ODE trajectories that a
continuous normalizing flow would otherwise be asked to learn directly
from likelihoods.

\textbf{Diffusion Models
\citep{sohl2015deep,ho2020denoising,song2021scorebased}.}
Diffusion models specify a forward stochastic process that progressively
corrupts data into a tractable reference distribution, and learn the
corresponding reverse process by score matching or denoising regression
\citep{sohl2015deep,ho2020denoising}. The continuous-time formulation of
\citep{song2021scorebased} unifies these constructions through a stochastic
differential equation, with sampling carried out by a reverse-time SDE or
its associated probability-flow ODE. The probability path is implicitly
prescribed by the forward noising schedule, and standard choices such as
the variance-preserving and variance-exploding processes induce specific
families of marginal paths between data and noise. Our framework relates
to the deterministic side of this picture: the trigonometric path of
\cref{ex:harmonic_to_trig} matches the marginal of the variance-preserving
diffusion path with a cosine schedule
\citep{lipman2023flow,ho2020denoising}, and arises in our framework as the
harmonic-oscillator least-action trajectory at $\omega = \pi/2$. This
provides a least-action interpretation of the variance-preserving schedule
at the level of the marginal probability path, although the full diffusion
model additionally specifies a stochastic forward process that lies outside
the deterministic continuity-equation framework developed here. Stochastic
extensions, in the spirit of Schr\"odinger bridges
(\cref{app:schrodinger}), are a natural direction for connecting the two
frameworks more closely.

\textbf{Conditional, Rectified, and OT-based Flow Matching
\citep{lipman2023flow,liu2023flow,tong2024improving}.}
Flow matching trains a neural velocity field by regression against a target
velocity associated with a prescribed probability path between
$p_{\mathrm{init}}$ and $p_{\mathrm{data}}$, avoiding the expensive ODE
solves required by likelihood-based training of continuous normalizing flows
\citep{chen2018neural,grathwohl2018scalable}. The framework is determined
by the choice of probability path and its associated velocity field.
\citep{lipman2023flow} introduce conditional flow matching with affine
Gaussian paths, in which conditioning on a target sample yields a
closed-form linear interpolation and a tractable per-sample regression
objective. \citep{liu2023flow} develop the rectified flow construction,
which interpolates linearly between independently sampled endpoints and
iteratively straightens the resulting trajectories. OT-CFM
\citep{tong2024improving} replaces the independent endpoint coupling by a
mini-batch optimal transport coupling under squared-Euclidean cost,
producing straight-line trajectories along the matched pairs and improving
sample efficiency at inference. These constructions share a common
geometric structure: endpoints are matched by a coupling and connected by
straight-line trajectories. Our framework extends this design space by
replacing the straight-line trajectory with the least-action trajectory of
a general Lagrangian, while preserving the simulation-free regression
structure. The kinetic Lagrangian of \cref{ex:harmonic_to_ot} recovers
OT-CFM as a special case, and the affine construction of
\citep{lipman2023flow} arises as its conditional analogue with $z = x_1$
and $p_{\mathrm{init}} = \mathcal{N}(0, I)$. Non-kinetic Lagrangians yield
curved least-action trajectories outside the affine and OT-CFM design
space, with the harmonic family of \cref{ex:harmonic_flow} as a
concrete instance.

\textbf{Stochastic Interpolants
\citep{albergo2023building,albergo2025stochastic}.}
Stochastic interpolants construct a probability path by prescribing a
pointwise interpolation $x_t = I(t, x_0, x_1) + \gamma(t)\, z$ between
coupled endpoints $(x_0, x_1)$, with an optional latent noise $z$ governed
by a schedule $\gamma(t)$. The interpolant $I$ and schedule $\gamma$ are
chosen by hand, typically as trigonometric or affine functions, and the
velocity field is then defined as the conditional expectation of
$\partial_t x_t$ given $x_t$. This framework unifies flow- and
diffusion-based generative models under a common interpolation viewpoint
and accommodates several specific path constructions, including affine and
trigonometric variance-preserving interpolants, by direct specification of
$I$ and $\gamma$. Our framework differs in how the interpolant itself is selected. Rather than prescribing
$I$ directly, we obtain it as the least-action trajectory of a Lagrangian
$\mathcal{L}(x, \dot{x}, t)$, with the endpoint coupling determined by the
induced optimal transport problem. The trigonometric interpolant of \cref{ex:harmonic_to_trig} appears in the
stochastic-interpolant framework as a specific hand-picked choice of $I$
and $\gamma$; in our framework it instead naturally arises as the harmonic-oscillator
least-action trajectory at $\omega = \pi/2$, embedded in a continuous
one-parameter family $\{\gamma^{\omega}_{x_0,x_1}\}_{\omega \in (0, \pi)}$
that smoothly deforms the free-particle straight line into curved
trajectories. More general Lagrangians produce interpolants outside the
stochastic-interpolant design space, including paths whose curvature and
endpoint coupling jointly reflect a prescribed potential. In this sense, the two frameworks are complementary: stochastic
interpolants offer flexibility through the direct specification of $I$ and
$\gamma$ together with stochastic forward processes, whereas Lagrangian
flow matching provides a principled mechanism for selecting deterministic
interpolants from a physical least-action principle, with stochastic
extensions left to future work.

\subsection{Action-based Methods for Learning Transport Dynamics}
\label{app:related_work_action}

Several recent works have explored connections between optimal transport,
least-action principles, and learning-based models for transport dynamics
\citep{neklyudov2023action,neklyudov2024a,pooladian2024neural,koshizuka2022neural}.
These works share the use of variational or action-based principles to
learn or infer transport dynamics. Our work follows this broad direction
but preserves the standard flow-matching learning structure: once the
training targets are constructed, the neural network simply learns to
approximate the velocity field. The distinguishing feature is that the
prescribed interpolation used in standard flow matching is replaced by one
induced from a Lagrangian: the Lagrangian defines endpoint costs, the
induced optimal transport problem selects the endpoint pairing, and the
least-action principle determines the trajectory between paired samples.
We discuss the connections to each of the four works above in turn.

\textbf{Action Matching \citep{neklyudov2023action}.}
Action Matching learns continuous-time dynamics from samples of temporal
marginal distributions $\{p_{t_k}\}_{k=1}^{K}$ along an absolutely continuous
probability path. The method parameterizes a scalar action potential
$s_\theta(x, t)$ and identifies its spatial gradient $\nabla s_\theta$ with
the velocity field that transports the prescribed marginals through the
continuity equation $\partial_t p_t + \nabla \cdot (p_t \nabla s_\theta) = 0$,
fitting $\theta$ through a variational objective on the potential. Two
features distinguish this setup from ours. First, the probability path is
prescribed: Action Matching observes intermediate marginals and recovers a
velocity field consistent with them, whereas in Lagrangian flow matching
only the endpoints $p_{\mathrm{init}}$ and $p_{\mathrm{data}}$ are given
and the Lagrangian determines the entire path. Second, our least-action
trajectories admit closed-form expressions in $(x_0, x_1, t)$ and serve as
explicit velocity targets in a standard flow-matching regression, in
contrast to a variational objective on a scalar potential.

\textbf{Wasserstein Lagrangian Flows \citep{neklyudov2024a}.}
Wasserstein Lagrangian Flows formulate action minimization directly on the
space of probability densities: a density path $\{p_t\}_{t \in [0, 1]}$ is
selected by minimizing a Lagrangian action functional
$\int_0^1 \mathcal{L}(p_t, \dot p_t, t)\, dt$ with marginal constraints at
observed time points. Appropriate choices of kinetic and potential terms
recover a wide range of transport problems as special cases, including
standard optimal transport, Schr\"odinger bridges, unbalanced transport,
and physically constrained transport. The corresponding computational
problem is an infinite-dimensional optimization over density paths, solved
through a Hamiltonian dual that introduces a co-state variable and reduces
to a min--max saddle-point problem. Our framework is less general at the
density-space level: we work with a separable classical Lagrangian
$\mathcal{L}(x, v, t) = K(v) - V(x)$ at the sample-trajectory level. In
exchange, the static--dynamic equivalence (\cref{thm:connection_pi_v_p})
converts the action into a closed-form endpoint cost and explicit
least-action trajectories between matched endpoints, so training reduces
to the usual flow-matching regression form without min--max optimization
or density-space dual variables.

\textbf{Neural Optimal Transport with Lagrangian Costs
\citep{pooladian2024neural}.}
Neural Optimal Transport with Lagrangian Costs studies the computational
problem of optimal transport when the ground cost is itself defined by a
least-action principle:
$c_\mathcal{L}(x_0, x_1) = \inf_\gamma \int_0^1 \mathcal{L}(\gamma(t), \dot\gamma(t), t)\, dt$
subject to $\gamma(0) = x_0$ and $\gamma(1) = x_1$, allowing the resulting
transport map and least-action paths to reflect obstacles, non-Euclidean
geometries, or other prior structure. The method parameterizes both the
transport map and the least-action paths as neural networks and trains
them jointly. This overlaps with the static component of our framework:
the cost $c_\mathcal{L}$ in our \cref{eq:cost_L} coincides with theirs in
the time-independent case. The focus is different, however. Their work
develops methods for \emph{computing} Lagrangian optimal transport maps as
the end product, with paths reported as a geometric byproduct. We instead
use the optimal coupling and least-action trajectories as an
\emph{intermediate step} in constructing conditional probability paths and
velocity targets for flow matching, with the simulation-free training
objective in \cref{eq:conditional_lfm_objective} as the end product. The
static--dynamic equivalence (\cref{thm:connection_pi_v_p}) lets us bypass
the need for an explicit transport map altogether.

\textbf{Neural Lagrangian Schr\"odinger Bridges
\citep{koshizuka2022neural}.}
Neural Lagrangian Schr\"odinger Bridges also use Lagrangian ideas to learn
transport dynamics, but model stochastic rather than deterministic dynamics
through a neural SDE $dX_t = b_\theta(X_t, t)\, dt + \sigma\, dW_t$ whose
drift encodes a Lagrangian action with kinetic and potential terms.
Training is formulated as a Schr\"odinger-bridge problem with marginal
constraints at observed time points: parameters are fit by simulating the
learned SDE and matching the resulting marginals to empirical samples,
which requires forward integration of the SDE inside the training loop and
is therefore not simulation-free in the flow-matching sense. Our framework
differs along three axes: the underlying dynamics are deterministic and
governed by the continuity equation, with a stochastic extension along the
lines of \cref{app:schrodinger} left to future work; the static--dynamic
equivalence (\cref{thm:connection_pi_v_p}) converts the Lagrangian action
into a closed-form per-pair training target; and training reduces to a
standard simulation-free regression objective on matched endpoint pairs,
preserving the scalability that motivates flow matching.

\section{Connection to Energy-based Methods}
\label{app:ebm}

We briefly discuss the relation between Lagrangian flow matching and
energy-based methods. Classical energy-based models (EBMs) assign a scalar
energy $E_\theta(x)$ to each state and define a distribution by favoring
low-energy configurations \cite{lecun2006tutorial,du2019implicit}; more
recent work \cite{balcerak2025energy} unifies energy-based modeling and
flow matching by learning a single scalar field that plays both energy and
drift roles.

Our connection to energy-based modeling is structural rather than
methodological. EBMs assign a preference to individual states; Lagrangians
assign a preference to entire trajectories,
\begin{equation*}
S[\gamma] \;=\; \int_0^1 \mathcal{L}\bigl(\gamma(t),\dot\gamma(t),t\bigr)\,dt,
\end{equation*}
lifting the same scalar-function-as-preference idea from the state space
$\mathbb{R}^d$ to the path space $\mathrm{AC}([0,1];\mathbb{R}^d)$. In the
separable case $\mathcal{L}(x,v,t) = K(v) - V(x)$, the kinetic term $K(v)$
penalizes rapid motion and the potential term $V(x)$ shapes which regions
of state space are preferred along the path. From this viewpoint, the
static--dynamic equivalence (Thm.~\ref{thm:connection_pi_v_p}) is the bridge
that pushes a path-level preference down to a state-level transport problem:
the optimal coupling and least-action trajectories between matched endpoints
are jointly selected by $\mathcal{L}$.

This perspective suggests a natural extension of our framework: rather than
prescribing the potential $V$, one could parameterize and learn it from data
through a family of Lagrangians
\begin{equation*}
\mathcal{L}_\theta(x,v,t) \;=\; K(v) - V_\theta(x,t),
\end{equation*}
where $V_\theta$ is a learnable potential. The induced least-action problem
would then adapt the path-space action itself within the flow-matching
framework, combining the simulation-free training of flow matching with the
energy-learning viewpoint of EBMs. We leave a full treatment to future work.

\section{Connection to Schr\"odinger Bridges}
\label{app:schrodinger}

The relation between Schr\"odinger bridges and optimal transport has been
extensively explored
\cite{leonard2013survey,chen2016relation,de2021diffusion,tong2024improving}.
In particular, Schr\"odinger bridges can be viewed as a stochastic,
entropy-regularized counterpart of optimal transport: in the Brownian
reference case the dynamic Schr\"odinger problem involves stochastic
transport with diffusion, whereas the optimal transport limit is governed
by deterministic transport dynamics.

Given endpoint distributions $p_0$ and $p_1$, the Schr\"odinger bridge
problem seeks a stochastic process whose initial and terminal marginals
match these endpoints while remaining close, in relative entropy, to a
reference process. If $P^{\mathrm{ref}}$ denotes the path measure of a
reference process, such as a Wiener process initialized from $p_0$, the
problem can be written as
\begin{equation*}
\inf_{P\,:\;(e_0)_\# P = p_0,\; (e_1)_\# P = p_1}
\mathrm{KL}\!\left(P \,\big\|\, P^{\mathrm{ref}}\right),
\end{equation*}
where $e_t$ denotes the evaluation map at time $t$. Lagrangian flow matching,
by contrast, is based on deterministic flow dynamics described by the
continuity equation, and minimizes an action functional over probability
paths and velocity fields,
\begin{equation*}
\inf_{(p,v)}
\int_0^1 \!\!\int_{\mathbb{R}^d}
\mathcal{L}\bigl(x, v(x,t), t\bigr)\, p_t(x) \, dx \, dt,
\end{equation*}
subject to endpoint constraints $p_{t=0} = p_0$ and $p_{t=1} = p_1$. The two
formulations select transport dynamics by different principles: a
Schr\"odinger bridge minimizes a relative entropy on path measures, whereas
Lagrangian flow matching minimizes a least-action functional on
density--velocity pairs.

This comparison suggests a stochastic extension of Lagrangian flow matching.
In place of the continuity equation, one could impose a Fokker--Planck
equation associated with stochastic dynamics driven by drift and diffusion.
Such an extension would replace deterministic transport by a stochastic
counterpart, moving Lagrangian flow matching closer to Schr\"odinger-bridge
and diffusion-based models.
\section{Proofs} \label{app:_b}
This appendix collects the proofs of the two theoretical results in
the main text. \Cref{app:_d1} proves the static--dynamic equivalence
(Theorem~\ref{thm:connection_pi_v_p}) underlying the Lagrangian flow
matching framework; \Cref{app:_d2} proves the conditional action
bound (Proposition~\ref{prop:conditional_action_dominates})
underlying the conditional formulation in
\Cref{sec:conditional_lfm}. Both proofs follow the standard
optimal-transport playbook: the forward direction is established by
constructing an admissible candidate from a static optimum and
bounding its action via Jensen's inequality, while the reverse
direction lifts an admissible density--velocity pair to a path
measure via the superposition principle of
\citep{ambrosio2005gradient}.

\subsection{Proof of Theorem \ref{thm:connection_pi_v_p}.}
\label{app:_d1}

The proof proceeds in two steps. We first construct an admissible
density--velocity pair from an optimal static coupling and show
that its action is bounded by $\mathcal A_{\mathrm{stat}}$, giving
$\mathcal A_{\mathrm{dyn}} \le \mathcal A_{\mathrm{stat}}$.
We then show the reverse inequality by lifting any admissible
$(p,v)$ to a path measure via the superposition principle and
bounding its action below by $\mathcal A_{\mathrm{stat}}$. The
construction of the forward direction is encapsulated in the
following lemma.

\begin{lemma}\label{lm:vel_cont_eqn}
Let \(\pi^*\) be an optimal plan of \eqref{eq:static_problem}, and let
\[
\eta^*:=(\gamma_{x_0,x_1})_\#\pi^*.
\]
Define
\begin{align*}
p^*(\cdot,t)&:=(e_t)_\#\eta^*,\\
v^*(x,t)&:=\mathbb E_{\gamma\sim\eta^*}\bigl[\dot\gamma(t)\mid \gamma(t)=x\bigr].
\end{align*}
Then \((p^*,v^*)\in\mathfrak D\).
\end{lemma}

\begin{proof}
First, since \(\pi^*\in\Pi(p_{\mathrm{init}},p_{\mathrm{data}})\), we have
\[
p^*(\cdot,0)=(e_0)_\#\eta^*=p_{\mathrm{init}},
\qquad
p^*(\cdot,1)=(e_1)_\#\eta^*=p_{\mathrm{data}}.
\]

It remains to verify the continuity equation in the sense of distributions.
Let \(\varphi\in C_c^\infty(\mathbb R^d\times(0,1))\). Since
\(p^*(\cdot,t)=(e_t)_\#\eta^*\), we have
\[
\int_{\mathbb R^d}\varphi(x,t)p^*(x,t)\,dx
=
\int_{AC([0,1];\mathbb R^d)}
\varphi(\gamma(t),t)\,d\eta^*(\gamma).
\]
Hence
\[
\int_0^1\!\!\int_{\mathbb R^d}
\partial_t\varphi(x,t)p^*(x,t)\,dx\,dt
=
\int_0^1\!\!\int_{AC([0,1];\mathbb R^d)}
\partial_t\varphi(\gamma(t),t)\,d\eta^*(\gamma)\,dt.
\]

Since \(\eta^*\) is concentrated on absolutely continuous curves,
the map \(t\mapsto \varphi(\gamma(t),t)\) is absolutely continuous for
\(\eta^*\)-a.e. \(\gamma\), and
\[
\frac{d}{dt}\varphi(\gamma(t),t)
=
\partial_t\varphi(\gamma(t),t)
+
\nabla\varphi(\gamma(t),t)\cdot\dot\gamma(t)
\]
for a.e. \(t\in(0,1)\). Therefore,
\[
\int_0^1 \frac{d}{dt}\varphi(\gamma(t),t)\,dt
=
\varphi(\gamma(1),1)-\varphi(\gamma(0),0).
\]
Since \(\varphi\) has compact support in \((0,1)\) in time, the right-hand side
vanishes. Integrating over \(\eta^*\), we obtain
\[
\int_0^1\!\!\int_{AC([0,1];\mathbb R^d)}
\partial_t\varphi(\gamma(t),t)\,d\eta^*(\gamma)\,dt
=
-
\int_0^1\!\!\int_{AC([0,1];\mathbb R^d)}
\nabla\varphi(\gamma(t),t)\cdot\dot\gamma(t)\,d\eta^*(\gamma)\,dt.
\]

By definition,
\[
p^*(\cdot,t)=(e_t)_\#\eta^*,
\]
so for any measurable vector field $F:\mathbb R^d\to\mathbb R^d$,
\[
\int_{AC([0,1];\mathbb R^d)}
F(\gamma(t))\cdot\dot\gamma(t)\,d\eta^*(\gamma)
=
\int_{\mathbb R^d}
F(x)\cdot
\mathbb E_{\gamma\sim\eta^*}\bigl[\dot\gamma(t)\mid\gamma(t)=x\bigr]
\,p^*(x,t)\,dx.
\]
Taking
\[
F(x)=\nabla\varphi(x,t)
\]
and using the definition
\[
v^*(x,t)
=
\mathbb E_{\gamma\sim\eta^*}\bigl[\dot\gamma(t)\mid\gamma(t)=x\bigr],
\]
we obtain
\[
\int_{AC([0,1];\mathbb R^d)}
\nabla\varphi(\gamma(t),t)\cdot\dot\gamma(t)\,d\eta^*(\gamma)
=
\int_{\mathbb R^d}
\nabla\varphi(x,t)\cdot v^*(x,t)p^*(x,t)\,dx.
\]
Therefore
\[
\int_0^1\!\!\int_{\mathbb R^d}
\Bigl(
\partial_t\varphi(x,t)
+
\nabla\varphi(x,t)\cdot v^*(x,t)
\Bigr)
p^*(x,t)\,dx\,dt
=
0,
\]
which proves
\[
\partial_t p^*+\operatorname{div}(v^*p^*)=0
\]
in the sense of distributions. Hence \((p^*,v^*)\in\mathfrak D\).
\end{proof}

\begin{proof}[Proof of Theorem \ref{thm:connection_pi_v_p}]

We first prove $\mathcal A_{\mathrm{dyn}}
\le \mathcal A_{\mathrm{stat}}$.

Let \(\pi^*\) be an optimal plan of \eqref{eq:static_problem}.
Since the minimizer \(\gamma_{x_0,x_1}\) of \eqref{eq:path_min_action}
is unique and the map is measurable, we can define the pushforward probability measure on the path space:
\[
\eta^*:=(\gamma_{x_0,x_1})_\#\pi^*.
\]
For each \(t\in[0,1]\), define
\[
p^*(\cdot,t):=(e_t)_\#\eta^*.
\]
Then
\[
p^*(\cdot,0)=p_{\mathrm{init}},
\qquad
p^*(\cdot,1)=p_{\mathrm{data}}.
\]
Since \(\eta^*\) is concentrated on absolutely continuous curves,
\(\dot\gamma(t)\) exists for \(\eta^*\)-a.e.\ \(\gamma\) and a.e.\ \(t\).
We define the Eulerian velocity field \(v^*\) by
\[
v^*(x,t)
:=
\mathbb E_{\gamma\sim\eta^*}\bigl[\dot\gamma(t)\mid \gamma(t)=x\bigr].
\]
By Lemma~\ref{lm:vel_cont_eqn}, the pair \((p^*,v^*)\) satisfies
\[
\partial_t p^*+\operatorname{div}(v^*p^*)=0
\]
in the sense of distributions, and therefore
\((p^*,v^*)\in\mathfrak D\).

We now estimate the action of \((p^*,v^*)\).
Since \(\mathcal L(x,\cdot,t)\) is convex, Jensen's inequality yields
\[
\mathcal L\!\left(x,v^*(x,t),t\right)
=
\mathcal L\!\left(
x,
\mathbb E_{\gamma\sim\eta^*}\!\left[\dot\gamma(t)\mid \gamma(t)=x\right],
t
\right)
\le
\mathbb E_{\gamma\sim\eta^*}\!\left[
\mathcal L(x,\dot\gamma(t),t)\mid \gamma(t)=x
\right]
\]
Multiplying by \(p^*(x,t)\) and integrating in \(x\), we obtain
\[
\int_{\mathbb R^d}\mathcal L(x,v^*(x,t),t)\,p^*(x,t)\,dx
\le
\int_{\mathbb R^d}
\mathbb E_{\gamma\sim\eta^*}\!\left[
\mathcal L(x,\dot\gamma(t),t)\mid \gamma(t)=x
\right]p^*(x,t)\,dx.
\]
By the defining property of conditional expectation,
\[
\int_{\mathbb R^d}
\mathbb E_{\gamma\sim\eta^*}\!\left[
\mathcal L(x,\dot\gamma(t),t)\mid \gamma(t)=x
\right]p^*(x,t)\,dx
=
\int_{AC([0,1];\mathbb R^d)}
\mathcal L(\gamma(t),\dot\gamma(t),t)\,d\eta^*(\gamma).
\]
Therefore,
\[
\int_{\mathbb R^d}\mathcal L(x,v^*(x,t),t)\,p^*(x,t)\,dx
\le
\int_{AC([0,1];\mathbb R^d)}
\mathcal L(\gamma(t),\dot\gamma(t),t)\,d\eta^*(\gamma).
\]
Integrating over \(t\in[0,1]\), it follows that
\[
\int_0^1\int_{\mathbb R^d}
\mathcal L(x,v^*(x,t),t)\,p^*(x,t)\,dx\,dt
\le
\int_{AC([0,1];\mathbb R^d)}
\left(
\int_0^1 \mathcal L(\gamma(t),\dot\gamma(t),t)\,dt
\right)d\eta^*(\gamma).
\]

By construction of \(\eta^*\), for \(\pi^*\)-a.e.\ \((x_0,x_1)\),
\(\gamma_{x_0,x_1}\) minimizes \eqref{eq:path_min_action}, so
\[
\int_0^1
\mathcal L(\gamma_{x_0,x_1}(t),\dot\gamma_{x_0,x_1}(t),t)\,dt
=
c_{\mathcal L}(x_0,x_1).
\]
Therefore
\[
\int_{AC([0,1];\mathbb R^d)}
\left(
\int_0^1
\mathcal L(\gamma(t),\dot\gamma(t),t)\,dt
\right)d\eta^*(\gamma)
=
\int_{\mathbb R^d\times\mathbb R^d}
c_{\mathcal L}(x_0,x_1)\,d\pi^*(x_0,x_1).
\]
Since \(\pi^*\) is optimal for \eqref{eq:static_problem},
\[
\int_{\mathbb R^d\times\mathbb R^d}
c_{\mathcal L}(x_0,x_1)\,d\pi^*(x_0,x_1)
=
\mathcal A_{\mathrm{stat}}.
\]
Thus
\[
\int_0^1\int_{\mathbb R^d}
\mathcal L(x,v^*(x,t),t)\,p^*(x,t)\,dx\,dt
\le
\mathcal A_{\mathrm{stat}},
\]
and hence $\mathcal A_{\mathrm{dyn}}
\le \mathcal A_{\mathrm{stat}}$.

We next prove the reverse inequality $\mathcal A_{\mathrm{dyn}} \ge\mathcal A_{\mathrm{stat}}$.

Let \((p,v)\in\mathfrak D\) have finite action. Since \(V\) is controlled by
the kinetic term in the sense stated above, finite action implies finite
kinetic energy. Therefore, using the positive definiteness of \(K\),
\[
\mathbb E_{(x,t)\sim p}\|v(x,t)\|^2
\le
C\mathbb E_{(x,t)\sim p}K(v(x,t))
<\infty .
\]
By \cite[Theorem 8.2.1]{ambrosio2005gradient}, there exists a probability
measure \(\eta\) concentrated on \(AC([0,1];\mathbb R^d)\) such that
\[
(e_t)_\#\eta=p(\cdot,t)
\quad\text{for all }t\in[0,1],
\qquad
\dot\gamma(t)=v(\gamma(t),t)
\quad
\eta\text{-a.e.}\;\; \gamma
\]
Define $\pi:=(e_0,e_1)_\#\eta\in
\Pi(p_{\mathrm{init}},p_{\mathrm{data}})$. Then
\[
\int_0^1\int_{\mathbb R^d}
\mathcal L(x,v(x,t),t)p(x,t)\,dx\,dt
=
\int_{AC([0,1];\mathbb R^d)}
\left(
\int_0^1
\mathcal L(\gamma(t),\dot\gamma(t),t)\,dt
\right)d\eta(\gamma).
\]
By definition of the least-action cost,
\[
\int_0^1
\mathcal L(\gamma(t),\dot\gamma(t),t)\,dt
\ge
c_{\mathcal L}(\gamma(0),\gamma(1)).
\]
Therefore
\[
\int_0^1\int_{\mathbb R^d}
\mathcal L(x,v(x,t),t)p(x,t)\,dx\,dt
\ge
\int_{\mathbb R^d\times\mathbb R^d}
c_{\mathcal L}(x_0,x_1)\,d\pi(x_0,x_1)
\ge
\mathcal A_{\mathrm{stat}}.
\]
Taking the infimum over all admissible pairs
\((p,v)\in\mathfrak D\) gives
\[
\mathcal A_{\mathrm{dyn}}
\ge
\mathcal A_{\mathrm{stat}}.
\]
This completes the proof.

\end{proof}

\subsection{Proof of Proposition \ref{prop:conditional_action_dominates}}
\label{app:_d2}
\begin{proof}
The proof has two steps. We first show that the marginal pair
$(p_t, v_t)$ induced by any admissible conditional family is itself
admissible for the unconditional problem; the inequality then follows
from Jensen applied to the convexity of $\mathcal L(x,\cdot,t)$.

Let \(\{p_t(\cdot\mid z),v_t(\cdot\mid z)\}_z\) be any admissible conditional
family, and let \((p_t,v_t)\) be the induced marginal path and velocity field defined by
\begin{equation}
p_t(x)
=
\int p_t(x\mid z)\,q(dz),\qquad
v_t(x)
=
\int
v_t(x\mid z)
\frac{p_t(x\mid z)}{p_t(x)}
\,q(dz).
\end{equation}
Integrating the continuity equations for $p_t(\cdot | z)$ over \(z\), we obtain
\begin{align}
\partial_t p_t(x)
+
\div\bigl(v_t(x)p_t(x)\bigr)
=
0.    
\end{align}
Moreover, the endpoint mixture conditions
\eqref{eq:conditional_endpoint_mixtures} imply
\[
p_0=p_{\mr{init}},
\qquad
p_1=p_{\mr{data}}.
\]
Thus, the marginal pair \((p_t,v_t)\in\mathfrak D\) is admissible for the
unconditional Lagrangian flow matching problem; see
Section \ref{sec:lagrangian_ot}.

Since \(\mathcal L_z=\mathcal L\) and \(\mathcal L(x,\cdot,t)\) is
convex, Jensen's inequality gives, for \(p_t(x)>0\),
\[
\mathcal L\bigl(x,v_t(x),t\bigr)
\le
\int
\mathcal L\bigl(x,v_t(x\mid z),t\bigr)
\frac{p_t(x\mid z)}{p_t(x)}
\,q(dz).
\]
Multiplying by \(p_t(x)\) and integrating over \(x\) and \(t\), we obtain
\[
\int_0^1
\int_{\mathbb R^d}
\mathcal L\bigl(x,v_t(x),t\bigr)p_t(x)\,dx\,dt
\le
\mathbb E_{z\sim q}
\int_0^1
\int_{\mathbb R^d}
\mathcal L\bigl(x,v_t(x\mid z),t\bigr)
p_t(x\mid z)\,dx\,dt .
\]
Since \((p_t,v_t)\) is admissible for the unconditional least-action problem,
\[
\mathcal A_{\mr{dyn}}
\le
\int_0^1
\int_{\mathbb R^d}
\mathcal L\bigl(x,v_t(x),t\bigr)p_t(x)\,dx\,dt .
\]
Combining the two inequalities and then taking the infimum over admissible
conditional families yields
\[
\mathcal A_{\mr{dyn}}
\le
\mathcal A_{\mr{cond}}.
\]

Finally, if the unconditional problem admits an optimizer
\((p_t^*,v_t^*)\), then equality is attained by the trivial conditioning
structure, namely by taking \(z\) to be constant and setting
\[
p_t(\cdot\mid z)=p_t^*,
\qquad
v_t(\cdot\mid z)=v_t^*.
\]
This proves the claim.
\end{proof}

\section{Mini-batch Conditional Harmonic Flow Matching}
\label{app:harmonic_flow_matching}

The following examples illustrate how the Lagrangian $\mathcal L$ and the batch size $n$ shape the resulting algorithm: the harmonic Lagrangian preserves the OT coupling but curves the training trajectories, mini-batch OT-CFM is recovered as the $\omega\to 0$ kinetic limit, and the batch size $n=1$ case collapses to independent-coupling flow matching.
\begin{example}[Mini-batch harmonic flow]
\label{ex:minibatch_harmonic}
The harmonic Lagrangian $\mathcal L_\omega$ of
\Cref{ex:harmonic_oscillator} yields the harmonic trajectory
\eqref{eq:harmonic_gamma} between matched endpoints, with cost
\[
c_\omega(x_0^{(i)},x_1^{(j)})
\;=\; \frac{\omega}{2\sin\omega}
\bigl[\cos\omega\,(\|x_0^{(i)}\|^2+\|x_1^{(j)}\|^2)
- 2\,x_0^{(i)}\!\cdot\!x_1^{(j)}\bigr].
\]
The empirical marginals fix
$\sum_i \pi_{ij}\|x_0^{(i)}\|^2$ and $\sum_j \pi_{ij}\|x_1^{(j)}\|^2$
for any feasible $\pi$, so
\[
\sum_{i,j}\pi_{ij}\,c_\omega(x_0^{(i)},x_1^{(j)})
\;=\; -\,\frac{\omega}{\sin\omega}\sum_{i,j}\pi_{ij}\,
x_0^{(i)}\!\cdot\!x_1^{(j)}
\;+\; C(z),
\]
with $C(z)$ independent of $\pi$. By polarization, this is
equivalent to minimizing the quadratic cost
$\sum_{i,j}\pi_{ij}\tfrac12\|x_0^{(i)}-x_1^{(j)}\|^2$, so the
discrete coupling \eqref{eq:minibatch_discrete_ot} remains the
quadratic-cost mini-batch OT --- only the training trajectory in
\eqref{eq:minibatch_lfm_objective} changes. The curvature of the
matched-pair trajectories visible across all panels of
\Cref{fig:ot_batch_sweep} is the harmonic signature: it is set by
$\omega$ and is independent of $n$. Additionally, panels
\ref{fig:ot_batch_sweep_10}--\ref{fig:ot_batch_sweep_100} of
\Cref{fig:ot_batch_sweep} show the intermediate regime, where
increasing $n$ tightens the empirical coupling toward $\pi^*$ while
$\omega$ continues to control trajectory curvature.
$\hfill\blacksquare$
\end{example}
 
\begin{example}[Mini-batch OT-CFM as the $\omega\to 0$ limit]
\label{ex:minibatch_kinetic}
Taking the frequency to limit $\omega\to 0$ in \Cref{ex:minibatch_harmonic} recovers the
kinetic Lagrangian $\mathcal L=\tfrac12\|v\|^2$ of
\Cref{ex:free_particle}, with quadratic cost
$c_{\mathcal L}(x_0,x_1)=\tfrac12\|x_0-x_1\|^2$ and affine trajectory
$\gamma_{x_0,x_1}(t)=(1-t)x_0+tx_1$. Both
\eqref{eq:minibatch_discrete_ot} and
\eqref{eq:minibatch_lfm_objective} then reduce to mini-batch OT-CFM
\cite{tong2024improving,pooladian2023multisample}.
$\hfill\blacksquare$
\end{example}

\begin{example}[Independent-coupling flow matching as the $n=1$ case]
\label{ex:minibatch_n1}
The other degenerate limit is $n=1$, where the discrete OT problem
is trivial: the conditional coupling reduces to the independent
product $\delta_{x_0^{(1)}}\otimes\delta_{x_1^{(1)}}$, and
marginalizing over $z$ recovers the unconditional objective
\eqref{eq:training_obj} with the independent coupling
$p_{\mathrm{init}}\otimes p_{\mathrm{data}}$. \Cref{fig:ot_batch_sweep_1} illustrates this regime: every source mode emits trajectories spanning both target moons, the many-to-many routing characteristic of independent coupling. Additionally, with the kinetic Lagrangian, this further reduces to the straight-line conditional-flow-matching objective of \cite{lipman2023flow}. 
$\hfill\blacksquare$
\end{example}
 
Mini-batch harmonic flow matching thus reuses the OT-CFM solver while replacing straight-line targets by curved least-action targets controlled by $\omega$. We exploit this in \Cref{sec:experiments}, where we use the same mini-batch OT solver as OT-CFM and vary only $\omega$ (and the anisotropic potential $A$ in \Cref{app:anisotropic_harmonic}) to study the effect of the Lagrangian on the learned dynamics.

\section{Anisotropic Harmonic Flow Matching}
\label{app:anisotropic_harmonic}
 
This appendix develops the anisotropic harmonic case in detail and
reports synthetic experiments. The construction generalizes the
isotropic harmonic Lagrangian of \Cref{ex:harmonic_oscillator} by
replacing the scalar frequency $\omega$ with a symmetric
positive-definite matrix $A\in\mathbb R^{d\times d}$, allowing
different curvatures along different directions in $\mathbb R^d$.
 
\paragraph{Lagrangian, trajectory, and cost}
 
Let $A\in\mathbb R^{d\times d}$ be symmetric positive-definite with
spectral decomposition $A = Q\Lambda Q^\top$ and $\Lambda = \mathrm{diag}(\omega_1^2,\ldots,\omega_d^2)$, $\omega_k\in(0,\pi)$. Define the anisotropic
harmonic Lagrangian
\begin{equation}
\label{eq:aniso_lagrangian}
\mathcal L_A(x,v)
\;=\; \tfrac12\|v\|^2 - \tfrac12\,x^\top A x.
\end{equation}
The Euler--Lagrange equation $\ddot\gamma + A\gamma = 0$ decouples
under the change of variables $\tilde\gamma = Q^\top\gamma$ into
$d$ scalar oscillators, one per eigendirection. Solving each scalar
problem as in \Cref{ex:harmonic_oscillator} and transforming back
gives the closed-form least-action trajectory
\begin{equation}
\label{eq:aniso_gamma}
\gamma_{x_0,x_1}^{A}(t)
\;=\; \frac{\sin\bigl((1-t)\sqrt{A}\bigr)}{\sin\sqrt{A}}\,x_0
+ \frac{\sin\bigl(t\sqrt{A}\bigr)}{\sin\sqrt{A}}\,x_1,
\end{equation}
where the matrix functions are defined by $f(A) = Q\,f(\Lambda)\,Q^\top$
applied entrywise to the eigenvalues. The velocity is
\begin{equation}
\label{eq:aniso_gamma_dot}
\dot\gamma_{x_0,x_1}^{A}(t)
\;=\; -\,\frac{\sqrt{A}\cos\bigl((1-t)\sqrt{A}\bigr)}{\sin\sqrt{A}}\,x_0
+ \frac{\sqrt{A}\cos\bigl(t\sqrt{A}\bigr)}{\sin\sqrt{A}}\,x_1.
\end{equation}
Both reduce to \eqref{eq:harmonic_gamma} when $A = \omega^2 I$, and
to the affine interpolation as $A\to 0$.
 
\paragraph{Cost.}
Substituting \eqref{eq:aniso_gamma}--\eqref{eq:aniso_gamma_dot} into
the action and using the per-mode calculation from
\Cref{ex:harmonic_flow} gives
\begin{equation}
\label{eq:aniso_cost}
c_A(x_0,x_1)
\;=\; \tfrac12\,x_0^\top \Phi(A)\,x_0
+ \tfrac12\,x_1^\top \Phi(A)\,x_1
- x_0^\top \Psi(A)\,x_1,
\end{equation}
where
\begin{equation}
\label{eq:aniso_Phi_Psi}
\Phi(A) \;:=\; \sqrt{A}\cot\sqrt{A},
\qquad
\Psi(A) \;:=\; \sqrt{A}\,\csc\sqrt{A},
\end{equation}
both symmetric positive-definite for the eigenvalue range
$\omega_k\in(0,\pi)$. The isotropic limit $A=\omega^2 I$ recovers
$\Phi(A) = \omega\cot\omega\cdot I$ and $\Psi(A) = (\omega/\sin\omega)\,I$,
matching the prefactors in \Cref{ex:harmonic_flow}.
 
\paragraph{Optimal coupling: a generalized OT problem}
 
A salient feature of the anisotropic case is that the polarization
argument of \Cref{ex:harmonic_flow} no longer reduces the static
problem to quadratic-cost OT. For any
$\pi\in\Pi(p_{\mathrm{init}},p_{\mathrm{data}})$,
\[
\int x_0^\top \Phi(A) x_0\,d\pi
= \int x_0^\top \Phi(A) x_0\,p_{\mathrm{init}}(dx_0)
\;=\; \mathrm{tr}\bigl(\Phi(A)\,\Sigma_0\bigr) + \mu_0^\top\Phi(A)\mu_0,
\]
where $\mu_0$ and $\Sigma_0$ are the mean and covariance of
$p_{\mathrm{init}}$, and analogously for $x_1$ under
$p_{\mathrm{data}}$. Both quadratic terms in
\eqref{eq:aniso_cost} are therefore $\pi$-independent, so
\begin{equation}
\label{eq:aniso_static_reduction}
\int c_A(x_0,x_1)\,d\pi
\;=\; -\int x_0^\top \Psi(A)\, x_1\,d\pi \;+\; \mathrm{const}.
\end{equation}
Minimizing over $\pi$ is equivalent to maximizing the
\emph{generalized cross-correlation}
$\int x_0^\top \Psi(A) x_1\,d\pi$. By the polarization identity
in the metric induced by $\Psi(A)$,
$x_0^\top\Psi(A) x_1
= \tfrac12\bigl(x_0^\top\Psi(A)x_0
+ x_1^\top\Psi(A) x_1
- (x_0-x_1)^\top\Psi(A)(x_0-x_1)\bigr)$,
this is in turn equivalent to the weighted-quadratic OT problem
\begin{equation}
\label{eq:aniso_weighted_ot}
\pi_A^*
\;\in\;
\operatorname*{argmin}_{\pi\in\Pi(p_{\mathrm{init}},p_{\mathrm{data}})}
\int \tfrac12\,(x_0-x_1)^\top \Psi(A)\, (x_0-x_1)\,d\pi(x_0,x_1).
\end{equation}
Equivalently, $\pi_A^*$ is the standard quadratic OT plan between
the pushforwards of $p_{\mathrm{init}}$ and $p_{\mathrm{data}}$ under
the linear map $x\mapsto \Psi(A)^{1/2}x$. Unlike the isotropic case,
$\pi_A^*$ generally depends on $A$ when $\Psi(A)$ is not a scalar
multiple of identity, so anisotropic harmonic flow matching does not
in general reuse the OT-CFM coupling. The mini-batch approximation
\eqref{eq:minibatch_discrete_ot} carries over by replacing
$\tfrac12\|x_0-x_1\|^2$ with
$\tfrac12(x_0-x_1)^\top\Psi(A)(x_0-x_1)$, and is solved by the same
linear-assignment routines.
 
\paragraph{Choice of \texorpdfstring{$A$}{A} via PCA.}
We parameterize $A$ from a principal component analysis of the data. Let
$\Sigma_{\mathrm{data}} = U\,\Lambda_{\mathrm{data}}\,U^\top$
with $\Lambda_{\mathrm{data}} = \mathrm{diag}(\lambda_1,\ldots,\lambda_d)$
and $\lambda_1\ge\cdots\ge\lambda_d>0$. We align the eigenframe of $A$
with the PCA frame $U$ and assign per-mode frequencies as a
monotonically decreasing function of the principal variances:
\begin{equation}
\label{eq:aniso_A_pca}
A \;=\; U\,\mathrm{diag}\bigl(\omega_1^2,\ldots,\omega_d^2\bigr)\,U^\top,
\qquad
\omega_k \;=\; \omega_{\max}\!\left(\frac{\lambda_d}{\lambda_k}\right)^{\!\alpha},
\end{equation}
with hyperparameters $\omega_{\max}\in(0,\pi)$ and $\alpha\ge 0$. This
yields small $\omega_k$ along high-variance principal directions —
where data already provide a strong signal and the trajectory should
remain close to the affine interpolant — and larger $\omega_k$ along
low-variance directions, where the harmonic potential more aggressively
contracts deviations from the geodesic. The eigenvalue constraint
$\omega_k\in(0,\pi)$ is automatically satisfied by the choice
$\omega_{\max}<\pi$. The exponent $\alpha$ controls anisotropy:
$\alpha=0$ recovers the isotropic harmonic case
$A=\omega_{\max}^2 I$ of \Cref{ex:harmonic_flow}, while $\alpha=\tfrac12$
corresponds to $A\propto \Sigma_{\mathrm{data}}^{-1}$ rescaled to fit
the spectral cap, recovering the covariance-aligned parameterization
discussed informally above. In practice, we estimate $U$ and
$\{\lambda_k\}$ from a held-out subset of $p_{\mathrm{data}}$ and freeze
$A$ during training; if the ambient dimension is large, we apply this
construction on the top-$r$ principal subspace and set
$\omega_k = \omega_{\max}$ for the orthogonal complement, so that
$A$ acts only on the directions where data anisotropy is statistically
meaningful.

\begin{wrapfigure}[14]{r}{0.3\textwidth}
\vspace{-0.4in}
\centering
\includegraphics[width=\linewidth]{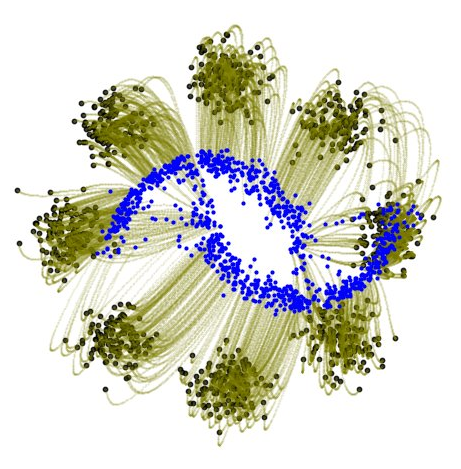}
\caption{OT-Aniso flow on the $8$-Gaussians$\to$two-moons benchmark.}
\label{fig:otaniso_8gaussian}
\end{wrapfigure}
\paragraph{Synthetic experiments.}
\Cref{fig:otaniso_8gaussian} visualizes the OT-Aniso flow on a 2D
toy benchmark in which eight Gaussian source modes (olive) surround
a two-moons target (blue). Olive curves show learned sample
trajectories integrated forward from each source cluster to the
target. The anisotropic harmonic potential bends each trajectory
along the data-subspace directions of largest variance, producing
the characteristic outward-then-inward sweep visible at every source
mode. We further evaluate OT-Aniso on CIFAR-10 image generation in
\Cref{tab:fid_nfe} and \Cref{fig:cifar10}. The anisotropic variant
OT-Aniso achieves $\mathrm{FID} = 4.078$ at adaptive integration,
no better than OT-SI and worse than every harmonic and OT-CFM
variant. We attribute this to the high-dimensional and largely
uncorrelated structure of CIFAR-10 pixel data: when most coordinates
are approximately independent, the rotation recovered from a small
fit batch carries little geometric signal beyond what an isotropic
schedule already provides, and the per-direction frequency assignment confers no meaningful advantage. Learning the potential $A$ directly from data, rather than fitting it from a single batch of empirical variances, is a natural direction for future work.
\section{Derivation of the normalized path energy}
\label{app:npe-derivation}

This appendix supplies the closed forms behind the diagnostic
$\mathrm{NPE}_\omega$ used in \Cref{sec:experiments}: the harmonic
geodesic between two endpoints, its kinetic energy in closed form, the
resulting harmonic-OT cost $C_\omega(\pi_0, \pi_1)$, and a decomposition
of $\mathrm{NPE}_\omega - 1$ into a coupling term and a path term that
isolate distinct sources of deviation.

\paragraph{Harmonic geodesics.}
A harmonic geodesic at frequency $\omega$ is a smooth path
$\phi:[0,1]\to\mathbb{R}^d$ satisfying $\ddot\phi_t + \omega^2\phi_t = 0$.
The general solution $\phi_t = a\cos(\omega t) + b\sin(\omega t)$,
specialized to the boundary conditions $\phi_0 = x_0$ and $\phi_1 = x_1$,
gives, for $\omega\in(0,\pi)$,
\begin{equation}
\label{eq:harmonic-geodesic}
\phi_t^\star(x_0, x_1) \;=\; \cos(\omega t)\,x_0
   + \sin(\omega t)\,B,
\qquad
B \;\coloneqq\; \frac{x_1 - \cos(\omega)\,x_0}{\sin(\omega)}.
\end{equation}
Differentiating yields
$\dot\phi_t^\star = -\omega\sin(\omega t)\,x_0 + \omega\cos(\omega t)\,B$.

\paragraph{Per-pair kinetic energy.}
The kinetic energy along $\phi^\star$ is
$k_\omega(x_0,x_1) = \int_0^1 \tfrac{1}{2}\|\dot\phi_t^\star\|^2\,dt$.
Expanding the square gives
\begin{equation}
\label{eq:kinetic-expansion}
\|\dot\phi_t^\star\|^2
= \omega^2\Big[
   \|x_0\|^2 \sin^2(\omega t)
   + \|B\|^2 \cos^2(\omega t)
   - 2\,(x_0\!\cdot\! B)\,\sin(\omega t)\cos(\omega t)
\Big].
\end{equation}
Integrating each term on $[0,1]$ produces the trigonometric integrals
\begin{equation}
\label{eq:trig-integrals}
I_{ss}(\omega) = \tfrac{1}{2} - \tfrac{\sin(2\omega)}{4\omega},
\qquad
I_{cc}(\omega) = \tfrac{1}{2} + \tfrac{\sin(2\omega)}{4\omega},
\qquad
I_{sc}(\omega) = \tfrac{\sin^2(\omega)}{2\omega},
\end{equation}
yielding the closed form
\begin{equation}
\label{eq:per-pair-kinetic}
k_\omega(x_0, x_1)
\;=\; \tfrac{\omega^2}{2}\Big[
   \|x_0\|^2\,I_{ss}(\omega)
 + \|B\|^2\,I_{cc}(\omega)
 - 2\,(x_0\!\cdot\! B)\,I_{sc}(\omega)
\Big].
\end{equation}
Equation~\eqref{eq:per-pair-kinetic} can be evaluated directly from
$(x_0, x_1)$ via $B$; we verify it numerically against Simpson-rule
integration of $\|\dot\phi^\star\|^2$ to $\sim\!10^{-15}$ relative error.

\paragraph{Limiting cases.}
Two limits provide useful sanity checks. As $\omega\to 0$,
$B\to x_1 - x_0$, $I_{ss}\to 0$, $I_{cc}\to 1$, and $I_{sc}\to 0$, so
\begin{equation}
\label{eq:limit-omega-zero}
k_\omega(x_0,x_1) \;\xrightarrow{\omega\to 0}\; \tfrac{1}{2}\|x_1 - x_0\|^2,
\end{equation}
recovering the linear-interpolant kinetic energy: harmonic geodesics
degenerate to straight lines, and $C_\omega \to W_2^2/2$. As
$\omega\to\pi^-$, $\sin(\omega)\to 0$ and the $B$-term diverges unless
$x_1 = -x_0$, reflecting the loss of geodesic uniqueness at the
period boundary; \eqref{eq:per-pair-kinetic} is therefore restricted
to $\omega\in(0,\pi)$.

\paragraph{Harmonic-OT cost.}
The harmonic transport cost between $\pi_0$ and $\pi_1$ is the infimum
of expected per-pair kinetic energy over couplings,
\begin{equation}
\label{eq:harmonic-ot}
C_\omega(\pi_0, \pi_1)
\;=\; \inf_{\pi\in\Pi(\pi_0,\pi_1)}\,
   \mathbb{E}_{(x_0,x_1)\sim\pi}\,k_\omega(x_0, x_1).
\end{equation}
Substituting~\eqref{eq:per-pair-kinetic} and using
$B = (x_1 - \cos(\omega)\,x_0)/\sin(\omega)$, the cross-term
$(x_0\!\cdot\! B) = (\langle x_0, x_1\rangle - \cos(\omega)\|x_0\|^2)/\sin(\omega)$
is the only coupling-dependent quantity (the marginal moments
$\mathbb{E}\|x_0\|^2,\mathbb{E}\|x_1\|^2$ are fixed). Hence the optimal
coupling for $C_\omega$ is the same coupling that maximizes
$\mathbb{E}\langle x_0, x_1\rangle$, which is the standard $W_2$ coupling
on centered measures. We estimate $C_\omega$ on minibatches by solving
the linear assignment problem with cost matrix
$M_{ij} = k_\omega(x_0^{(i)}, x_1^{(j)})$ via the Hungarian
algorithm~\citep{kuhn1955hungarian}.

\paragraph{Decomposition of $\mathrm{NPE}_\omega$.}
Let $\hat\pi_\theta$ denote the coupling induced by the learned flow
(each $x_0$ paired with the endpoint $\phi_1(x_0)$ produced by ODE
integration of $v_\theta$), and $\pi^\star$ the harmonic-OT coupling.
Define the model-coupling action
$\bar k_\theta = \mathbb{E}_{(x_0,x_1)\sim\hat\pi_\theta}\,k_\omega(x_0,x_1)$.
Then
\begin{equation}
\label{eq:decomposition}
K[\phi] - C_\omega(\pi_0,\pi_1)
\;=\;
\underbrace{\bar k_\theta - C_\omega(\pi_0,\pi_1)}_{\text{coupling excess}}
\;+\;
\underbrace{K[\phi] - \bar k_\theta}_{\text{path excess}}.
\end{equation}
The coupling excess is non-negative by definition of the OT infimum
and vanishes iff $\hat\pi_\theta = \pi^\star$; the path excess is signed
and vanishes iff each model trajectory traces the harmonic geodesic
between its endpoints. Both vanish iff $\phi$ realizes the harmonic-OT
geodesic flow at frequency $\omega$, in which case
$\mathrm{NPE}_\omega = 0$. The decomposition lets us attribute observed
NPE deviations to either an incorrect pairing (e.g.\ the independent
coupling $(x_0, z)$ used by vanilla harmonic flow) or non-geodesic
trajectories (e.g.\ rectified flow).

\paragraph{Estimation.}
We integrate the model ODE $\dot\phi_t = v_\theta(\phi_t,t)$ with a
fixed-step RK4 scheme at $N=200$ steps, recording
$\tfrac12\|v_\theta(\phi_t,t)\|^2$ on the same grid, and compute
$K[\phi]$ from these samples by composite Simpson's rule. The
harmonic-OT cost $C_\omega$ in \eqref{eq:harmonic-ot} is estimated by
solving the linear assignment problem with cost matrix
$M_{ij} = k_\omega(x_0^{(i)}, x_1^{(j)})$ on a fixed minibatch of size
$n = 512$, drawn independently of the RK4 evaluation batch. The
coupling and path components of \eqref{eq:decomposition} are computed
on the RK4 batch using the model-induced pairs
$(x_0,\,\phi_1(x_0))$. We report all NPE-based diagnostics at the
reference frequency $\omega = 1$, and aggregate as mean $\pm$ standard
deviation across five independent training runs (different seeds for
data sampling, model initialization, and optimization).
\section{Experimental Setup}
\label{app:exp_setup}

This appendix gives the experimental details deferred from
\Cref{sec:experiments}: training protocols, evaluation procedures, and
supplementary results for the three tracks (2D synthetic, single-cell
trajectory interpolation, and CIFAR-10 image generation). Common
settings shared across tracks are stated once in
\Cref{app:exp_setup_shared} and not repeated.

\subsection{Shared protocol}
\label{app:exp_setup_shared}

\paragraph{Framework and architectures.}
All experiments build on the conditional flow matching framework of
\citep{tong2024improving}: model architectures, optimizer settings,
training schedules, dataset preprocessing, and evaluation harnesses
follow that codebase, with the harmonic conditional flow matchers
swapping in for the OT-CFM matcher where appropriate. The 2D and
single-cell tracks use a time-conditioned MLP; the CIFAR-10 track
uses the U-Net configuration of \citep{tong2024improving}. We refer
to \citep{tong2024improving} for hyperparameter values; only
deviations or additions are listed below.

\paragraph{Methods.}
We compare six conditional flow matching variants. All share
architecture, optimizer, training-step count, and minibatch size
within a given track; differences arise solely from the conditional
probability path and source--target coupling.
\begin{itemize}[leftmargin=1.5em,itemsep=2pt,topsep=2pt]
\item \textbf{OT-CFM}~\citep{tong2024improving}: linear interpolant
$x_t = (1-t)\,x_0 + t\,x_1$ with the exact-OT minibatch coupling on
the squared-Euclidean cost.
\item \textbf{OT-SI}~\citep{albergo2023building}: variance-preserving
(cosine-schedule) interpolant with the exact-OT coupling on the
squared-Euclidean cost.
\item \textbf{OT-Harmonic ($\omega \in \{10^{-3},\, 1,\, \pi/2\}$)}:
the isotropic harmonic conditional flow of \Cref{sec:lagrangian_ot,sec:conditional_lfm} with
frequency $\omega$, coupled by exact minibatch OT on the harmonic
action cost $c_\omega$ (\Cref{ex:minibatch_harmonic}).
\item \textbf{OT-Aniso} (CIFAR-10 only): the anisotropic harmonic
flow of \Cref{app:anisotropic_harmonic}, coupled by exact minibatch
OT on the anisotropic action cost $c_A$.
\end{itemize}
The $\omega = 10^{-3}$ setting serves as a numerical proxy for the
$\omega \to 0$ limit of the harmonic family, which recovers the
affine interpolant and hence OT-CFM up to seed variation; we report
both for completeness.

\paragraph{Training and seeding.}
Optimal-transport couplings are computed per-minibatch by the
Hungarian algorithm~\citep{kuhn1955hungarian}. For OT-CFM and OT-SI
the cost is the squared-Euclidean distance. For isotropic OT-Harmonic
the cost is the harmonic action $c_\omega$, which differs from
squared-Euclidean by a $\pi$-independent constant under fixed
empirical marginals and so shares the same minimizer
(\Cref{ex:minibatch_harmonic}). For OT-Aniso the cost is the
anisotropic action $c_A$, which does not reduce to squared-Euclidean
and yields a genuinely different coupling. We use independent random
seeds for data sampling, parameter initialization, and OT-coupling
batches; results are reported as mean $\pm$ standard deviation across
five independent retrains.

\paragraph{Evaluation metrics.}
The 2-Wasserstein distance $W_2$ between generated and target samples
is estimated by linear assignment on a held-out batch of
$N_{\mathrm{eval}} = 2{,}048$ samples. The normalized path energy
$\mathrm{NPE}_\omega$ is computed as in \Cref{app:npe-derivation} on
a sub-batch of $n = 512$ source samples drawn from the
$N_{\mathrm{eval}}$ batch and $n = 512$ fresh target samples: model
trajectories are integrated from the source sub-batch by fixed-step
RK4 at $200$ steps, the kinetic-energy integrand
$\tfrac12\|v_\theta(\phi_t,t)\|^2$ is sampled on the same grid and
integrated by composite Simpson's rule, and the harmonic-OT cost
$C_\omega$ between source and fresh target sub-batches is estimated
by Hungarian assignment on the $512 \times 512$ pairwise
harmonic-action matrix.

\subsection{Two-dimensional synthetic data}
\label{app:exp_setup_2d}

\paragraph{Distributions.}
We use four sources and three targets, all in $\mathbb{R}^2$:
$\mathcal{N}(0, I)$ as the canonical noise source; the
\emph{8gaussians} mixture (eight isotropic Gaussians equally spaced
on a radius-5 circle, $\sigma = 0.5$); the \emph{moons} dataset (two
interlocking half-circles, additive noise $0.1$); and a 2D projection
of the standard \emph{S-curve} manifold (using coordinates $0$ and
$2$, standardized then rescaled by $7$ for visual contrast). The
four pairs reported in \Cref{tab:harmonic_comparison} are
$\mathcal{N}\!\to\!\text{8gaussians}$,
$\text{moons}\!\leftrightarrow\!\text{8gaussians}$,
$\mathcal{N}\!\to\!\text{moons}$, and
$\mathcal{N}\!\to\!\text{S-curve}$. Each configuration is trained for
$20{,}000$ Adam steps at minibatch size $256$.

\paragraph{Inference and the NFE plots.}
\Cref{tab:harmonic_comparison} in the main text reports
distribution-fit and path-energy diagnostics at a fixed inference
budget. To probe the inference-cost / sample-quality trade-off
directly, we additionally sweep
$\mathrm{NFE} \in \{4, 8, 16, 32, 64, 128\}$ across three fixed-step
ODE integrators of increasing order (\textsc{Euler}, \textsc{midpoint},
\textsc{RK4}). Since the higher-order solvers require more
vector-field evaluations per integration step, the actual number of
integration steps for a given NFE budget is $\mathrm{NFE} / k$ with
$k = 1, 2, 4$ for Euler, midpoint, and RK4 respectively; this
normalization plots all three curves on a common compute budget
rather than a common step count. \Cref{fig:solver} shows the
resulting curves for all five source$\to$target settings. Each panel
reports mean $W_2$ across five training seeds at
$N_{\mathrm{eval}} = 2{,}048$, with shaded $\pm 1$ standard-deviation
bands. The figure makes two observations from
\Cref{tab:harmonic_comparison} visible across the full inference
budget: (i) OT-CFM and OT-Harmonic at $\omega = 10^{-3}$ overlap in
every panel, confirming the affine-limit prediction numerically and
serving as a sanity check on the implementation; and (ii) OT-Harmonic
at $\omega = 1$ matches or improves on this shared baseline at every
NFE, with the largest gap in the low-NFE regime where step error from
non-straight trajectories dominates.

\begin{figure}[t]
\centering
\begin{subfigure}[t]{0.33\linewidth}
  \centering
  \includegraphics[width=\linewidth]{figures/w2_vs_nfe_8gaussian_to_moons_euler.png}
\end{subfigure}\hfill
\begin{subfigure}[t]{0.33\linewidth}
  \centering
  \includegraphics[width=\linewidth]{figures/w2_vs_nfe_8gaussian_to_moons_midpoint.png}
\end{subfigure}\hfill
\begin{subfigure}[t]{0.33\linewidth}
  \centering
  \includegraphics[width=\linewidth]{figures/w2_vs_nfe_8gaussian_to_moons_rk4.png}
\end{subfigure}

\begin{subfigure}[t]{0.33\linewidth}
  \centering
  \includegraphics[width=\linewidth]{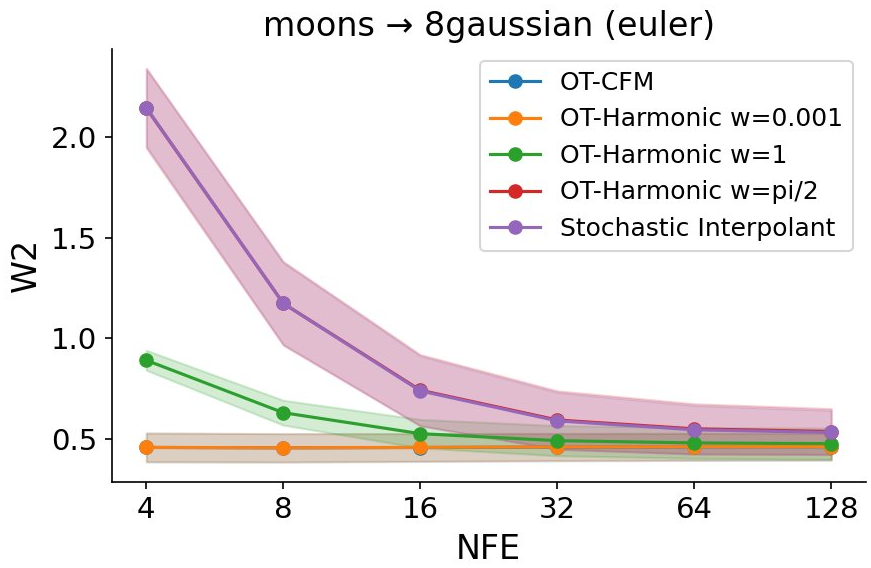}
\end{subfigure}\hfill
\begin{subfigure}[t]{0.33\linewidth}
  \centering
  \includegraphics[width=\linewidth]{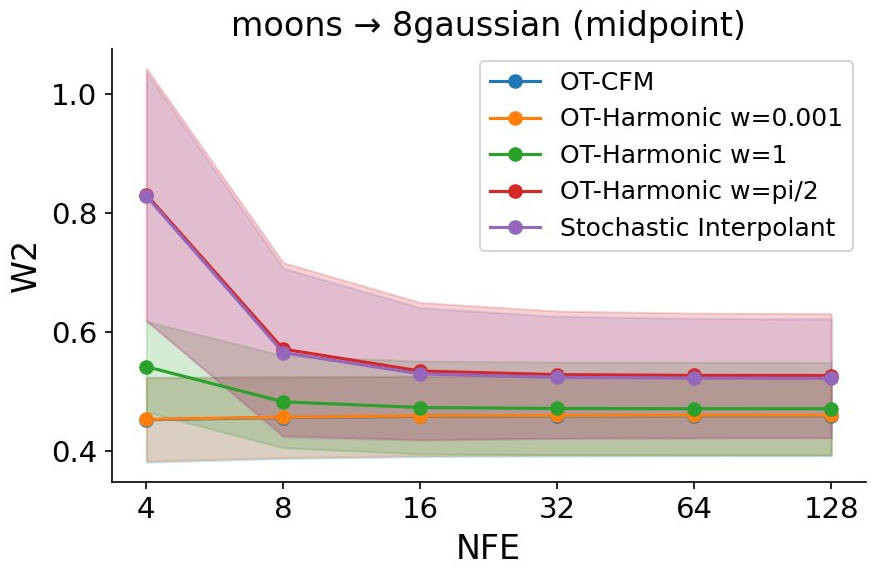}
\end{subfigure}\hfill
\begin{subfigure}[t]{0.33\linewidth}
  \centering
  \includegraphics[width=\linewidth]{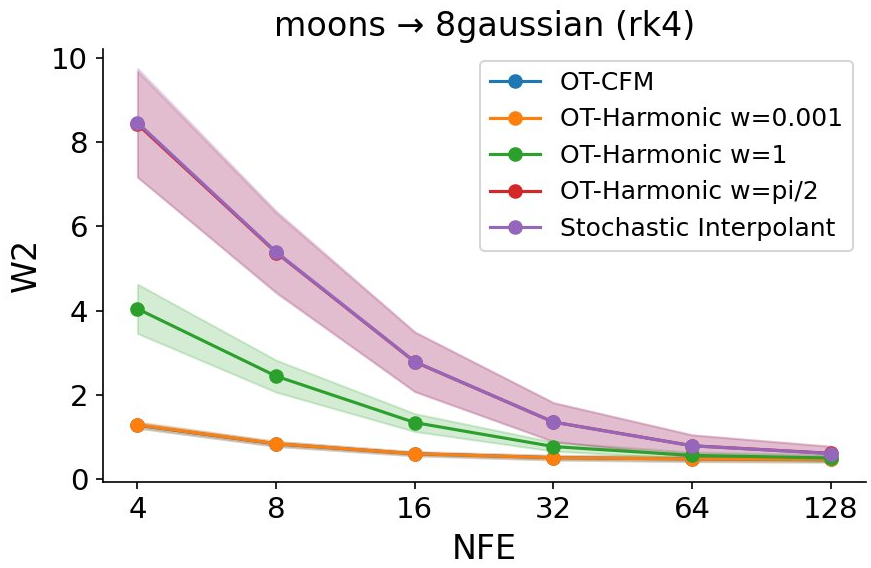}
\end{subfigure}

\begin{subfigure}[t]{0.33\linewidth}
  \centering
  \includegraphics[width=\linewidth]{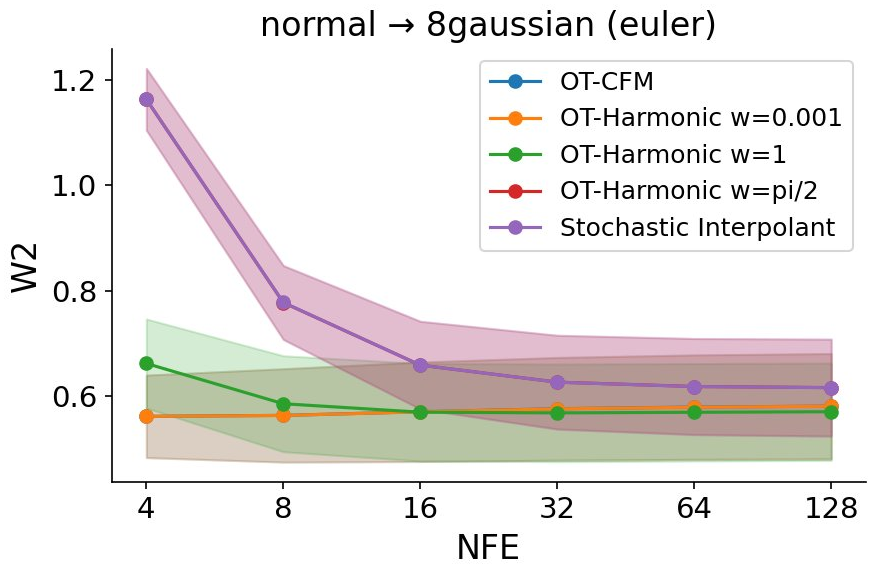}
\end{subfigure}\hfill
\begin{subfigure}[t]{0.33\linewidth}
  \centering
  \includegraphics[width=\linewidth]{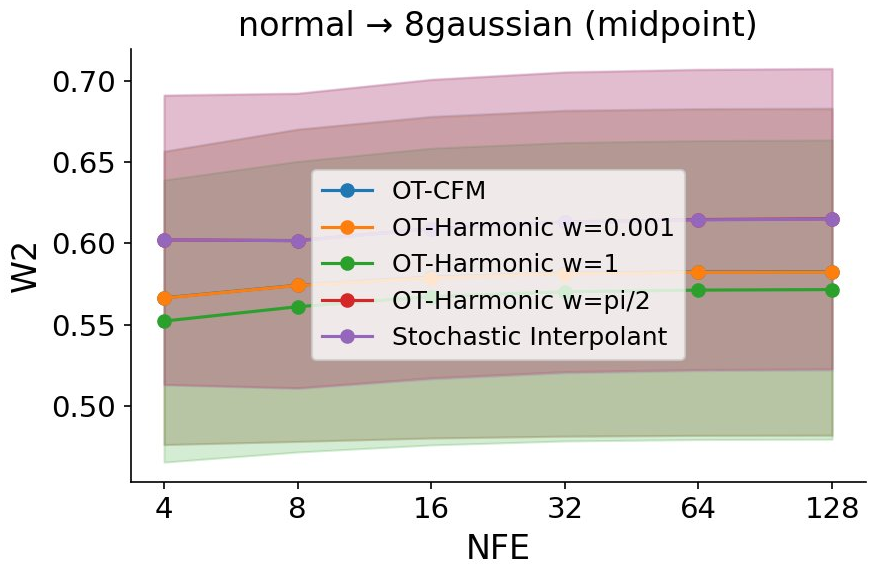}
\end{subfigure}\hfill
\begin{subfigure}[t]{0.33\linewidth}
  \centering
  \includegraphics[width=\linewidth]{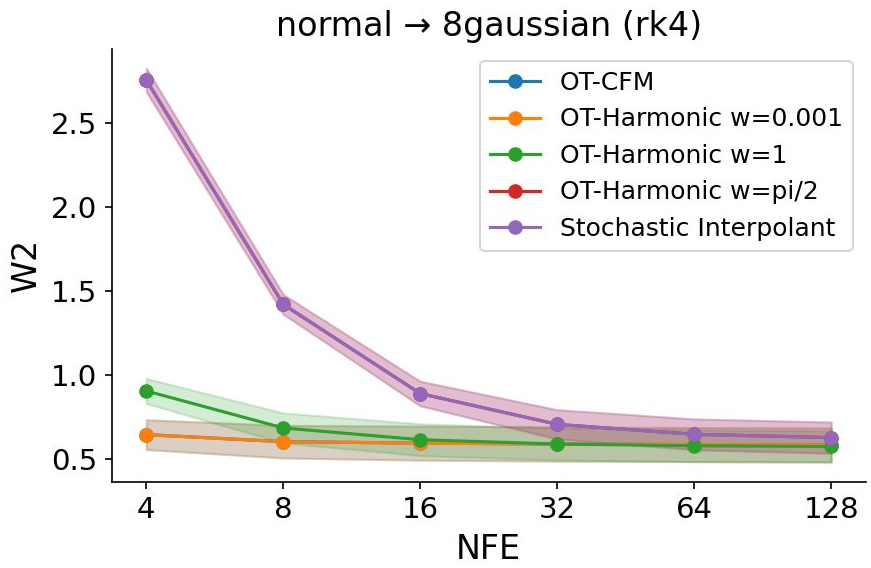}
\end{subfigure}

\begin{subfigure}[t]{0.33\linewidth}
  \centering
  \includegraphics[width=\linewidth]{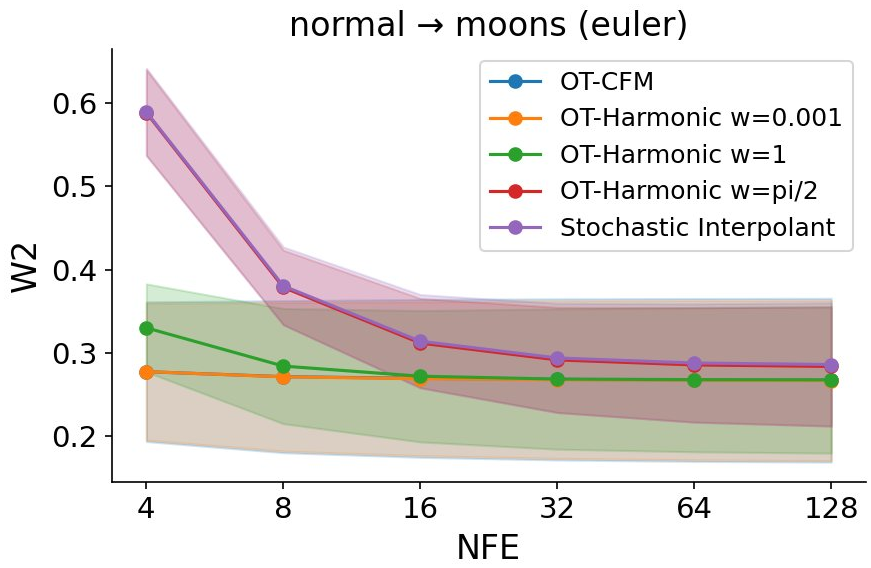}
\end{subfigure}\hfill
\begin{subfigure}[t]{0.33\linewidth}
  \centering
  \includegraphics[width=\linewidth]{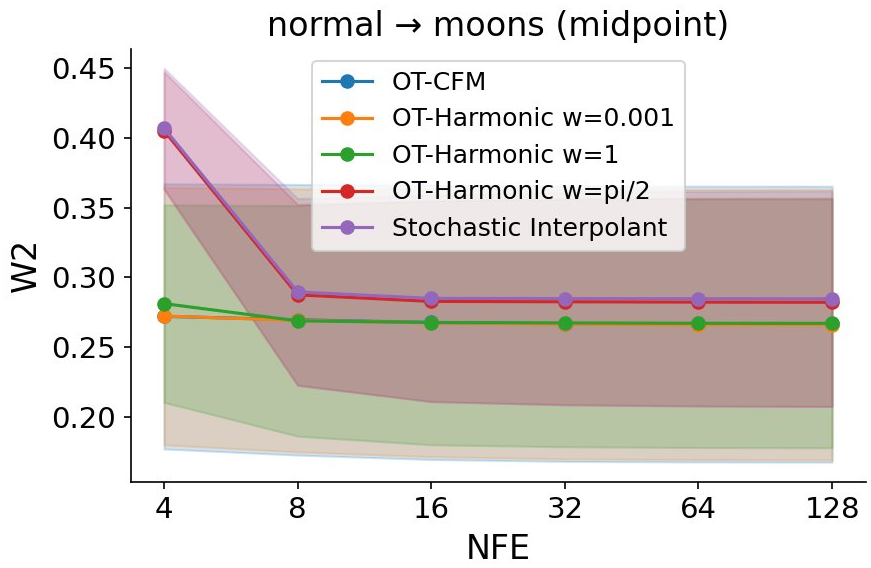}
\end{subfigure}\hfill
\begin{subfigure}[t]{0.33\linewidth}
  \centering
  \includegraphics[width=\linewidth]{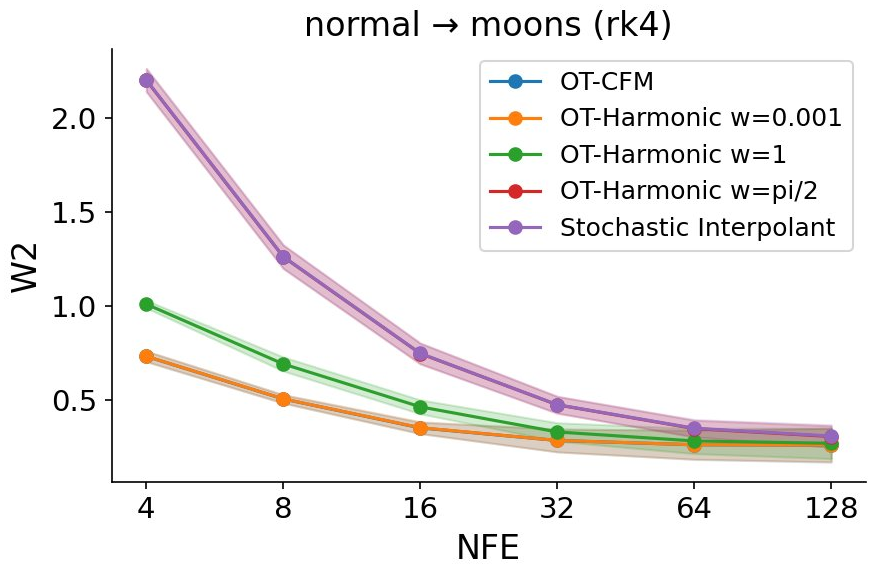}
\end{subfigure}

\begin{subfigure}[t]{0.33\linewidth}
  \centering
  \includegraphics[width=\linewidth]{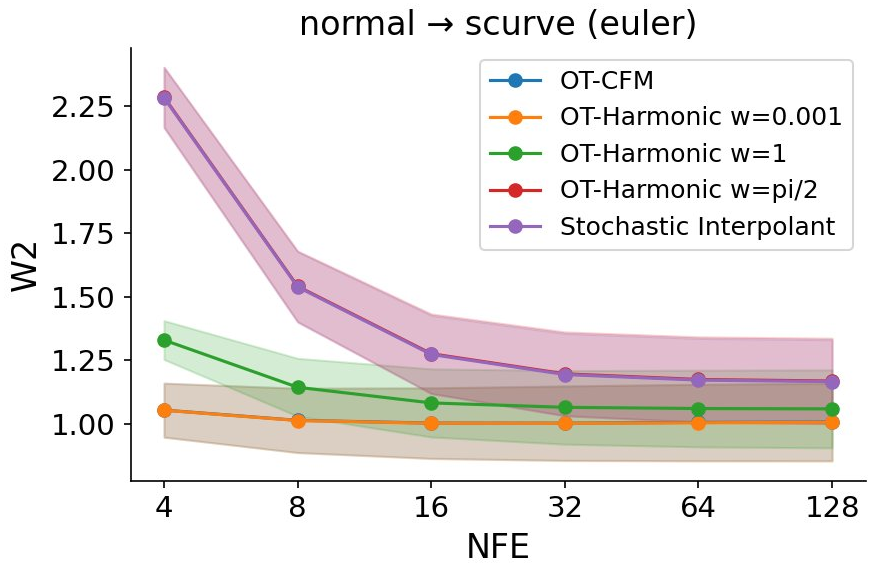}
\end{subfigure}\hfill
\begin{subfigure}[t]{0.33\linewidth}
  \centering
  \includegraphics[width=\linewidth]{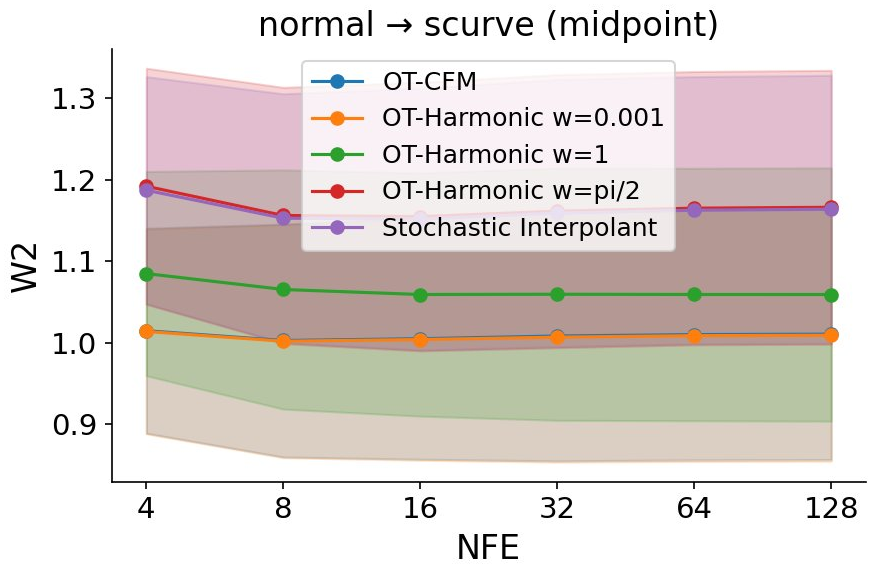}
\end{subfigure}\hfill
\begin{subfigure}[t]{0.33\linewidth}
  \centering
  \includegraphics[width=\linewidth]{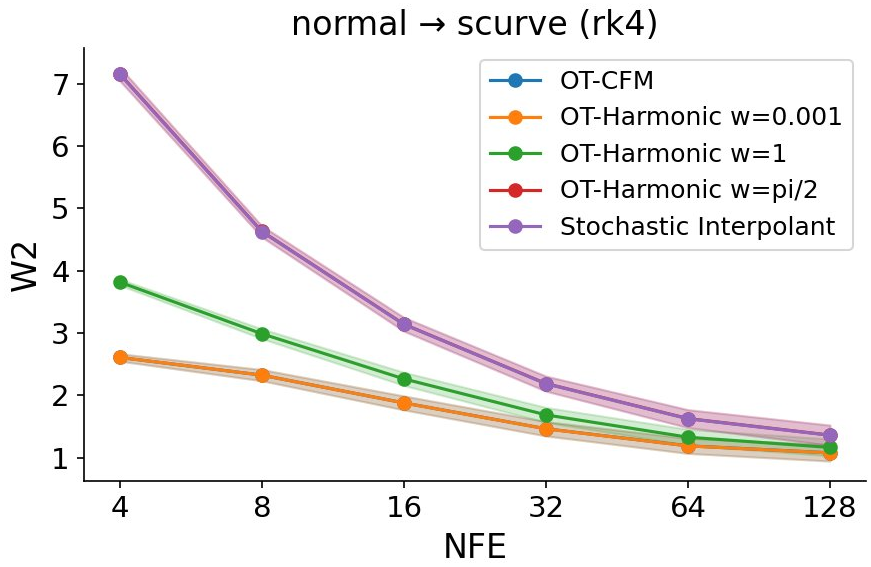}
\end{subfigure}
\caption{Sample quality vs.\ inference budget. We plot the 2-Wasserstein
distance $W_2$ between generated and target samples against the number
of function evaluations (NFE) at inference, across five
source$\to$target settings (rows) and three ODE integrators (columns).
OT-CFM and OT-Harmonic at $\omega = 10^{-3}$ are visually
indistinguishable, consistent with the harmonic interpolant degenerating
to the affine one as $\omega \to 0$. OT-Harmonic at $\omega = 1$ matches
or improves on this shared baseline at every NFE budget, with the
largest gains in the low-NFE regime where step error dominates. Shaded
bands show $\pm 1$ standard deviation across five training seeds.}
\label{fig:solver}
\end{figure}

\subsection{Single-cell trajectory interpolation}
\label{app:exp_setup_singlecell}

We follow the leave-one-out protocol of \citep{tong2024improving} on
three datasets: \emph{Embryoid body (EB)}~\citep{moon2019visualizing},
\emph{CITE-seq}, and \emph{Multiome}~\citep{burkhardt2022multimodal}.
For each dataset we hold out one intermediate timepoint $t$ at a
time, train on the union of remaining timepoints, and evaluate the
1-Wasserstein distance between the model-interpolated distribution at
$t$ and the held-out empirical distribution. \Cref{tab:single_cell_app}
reports the per-seed scalar (mean $W_1$ across held-out timepoints),
averaged as mean $\pm$ standard deviation across five training seeds
$\{42, 43, 44, 45, 46\}$.

\paragraph{Methods.}
We compare the five non-anisotropic methods of
\Cref{app:exp_setup_shared}: OT-CFM, OT-SI, and OT-Harmonic at
$\omega \in \{10^{-3}, 1, \pi/2\}$. We do not run OT-Aniso on this
track: PCA-based anisotropy fitting requires a sufficiently rich fit
batch, which the per-timepoint sample sizes do not provide.

\paragraph{Architecture and training.}
The velocity field is parameterized by an MLP with width $64$ and time-conditioned input. We train with AdamW at learning rate $10^{-3}$, weight decay $10^{-5}$, batch size $128$, and noise scale $\sigma = 0.1$. The training budget is matched to the runner protocol of
\citep{tong2024improving}: each epoch yields one minibatch per
trajectory segment over the $80\%$ training split, giving
$\mathrm{steps/epoch} = \min_t \lfloor 0.8\,|X_t| / \mathrm{batch} \rfloor$
optimizer steps; we run for $1{,}000$ epochs.

\paragraph{Evaluation.}
At test time we integrate from the earliest surviving timepoint to
the renumbered position of the held-out one with fixed-step Euler at
$100$ steps, then compute $W_1$ via exact linear assignment between
$1{,}000$ generated samples and $1{,}000$ samples from the held-out distribution.

\paragraph{Preprocessing.}
All three datasets are obtained from the public Mendeley release of
\citep{moon2019visualizing}. We use the top-$5$ principal components
of each dataset, standardized to zero mean and unit standard
deviation per coordinate. The CITE-seq and Multiome data were
originally released as part of the Open Problems Multimodal
Single-Cell Integration competition. For EB, we use the same PC
dimensionality ($r = 5$) as \citep{tong2024improving}, but the
resulting embedding is not identical to theirs due to a preprocessing
difference; absolute $W_1$ values on EB are therefore not directly
comparable to numbers reported in that work, while the relative
ordering between methods remains a meaningful comparison since all
five methods are evaluated on the same embedding.

\paragraph{Time-Stratified 1-Wasserstein Loss Across Held-Out Timepoints.}
\Cref{tab:single_cell_app} reports leave-one-out $1$-Wasserstein
distances at each held-out timepoint across the three single-cell
datasets. The harmonic interpolant in the small-$\omega$ regime
($\omega = 0.001$) is competitive with OT-CFM throughout: the two
methods are within $3 \times 10^{-3}$ of each other on every CITE and
EB timepoint, and OT-Harmonic ($\omega = 0.001$) attains the lowest
error on both CITE timepoints, on EB at $t=2$, and on both Multiome
timepoints. This is consistent with the theoretical limit in which the
harmonic flow recovers the OT-CFM trajectory as $\omega \to 0$.
Increasing $\omega$ degrades performance monotonically: $\omega = 1$
remains within roughly $5$--$15\%$ of the leaders but already loses on
most timepoints, and $\omega = \pi/2$ collapses to the OT-SI regime,
with the two methods producing essentially indistinguishable errors
across all seven columns (e.g., $1.1748$ vs.\ $1.1852$ on EB-$t=1$;
$1.8278$ vs.\ $1.8139$ on Multiome-$t=2$). The gap between the
harmonic and OT-CFM/OT-SI extremes widens with the prediction horizon:
on EB, all methods are tightly clustered at $t=2$, but by $t=3$ the
high-$\omega$ variants incur roughly $60\%$ more error than the
low-$\omega$ ones, suggesting that the oscillatory component of the
harmonic flow primarily hurts long-range extrapolation rather than
near-boundary interpolation.
\begin{table}[t]
\centering
\caption{Single-cell trajectory interpolation across three datasets, evaluated at each leave-one-out held-out timepoint. Entries are $1$-Wasserstein distance to the held-out distribution; bold indicates the lowest mean per column.}
\label{tab:single_cell_app}
\setlength{\tabcolsep}{4pt}
\renewcommand{\arraystretch}{1.15}
\resizebox{\linewidth}{!}{%
\begin{tabular}{l cc ccc cc}
\toprule
& \multicolumn{2}{c}{CITE} & \multicolumn{3}{c}{EB} & \multicolumn{2}{c}{Multiome} \\
\cmidrule(lr){2-3} \cmidrule(lr){4-6} \cmidrule(lr){7-8}
Algorithm & $t=1$ & $t=2$ & $t=1$ & $t=2$ & $t=3$ & $t=1$ & $t=2$ \\
\midrule
OT-CFM
 & $0.9262 \pm 0.035$ & $0.8720 \pm 0.035$
 & $\mathbf{0.9504 \pm 0.045}$ & $0.8865 \pm 0.045$ & $\mathbf{1.0187 \pm 0.140}$
 & $0.8574 \pm 0.060$ & $1.2868 \pm 0.080$ \\
OT-SI
 & $1.1196 \pm 0.035$ & $1.3651 \pm 0.035$
 & $1.1852 \pm 0.045$ & $1.1934 \pm 0.045$ & $1.6348 \pm 0.140$
 & $1.1573 \pm 0.060$ & $1.8139 \pm 0.080$ \\
\midrule
OT-Harmonic, $\omega = 0.001$
 & $\mathbf{0.9261 \pm 0.038}$ & $\mathbf{0.8709 \pm 0.038}$
 & $0.9510 \pm 0.045$ & $\mathbf{0.8846 \pm 0.045}$ & $1.0213 \pm 0.140$
 & $\mathbf{0.8562 \pm 0.060}$ & $\mathbf{1.2855 \pm 0.080}$ \\
OT-Harmonic, $\omega = 1$
 & $0.9338 \pm 0.030$ & $0.9969 \pm 0.030$
 & $0.9829 \pm 0.060$ & $0.8951 \pm 0.060$ & $1.0946 \pm 0.140$
 & $0.9017 \pm 0.080$ & $1.4106 \pm 0.080$ \\
OT-Harmonic, $\omega = \pi/2$
 & $1.1232 \pm 0.025$ & $1.3514 \pm 0.025$
 & $1.1748 \pm 0.045$ & $1.1524 \pm 0.045$ & $1.6785 \pm 0.140$
 & $1.1599 \pm 0.080$ & $1.8278 \pm 0.080$ \\
\bottomrule
\end{tabular}
}
\end{table}

\subsection{CIFAR-10 image generation}
\label{app:exp_setup_cifar}
All CIFAR-10 runs share a single training pipeline that selects the
appropriate flow matcher via a model flag. Architecture, optimizer,
and schedule follow \citep{tong2024improving} unchanged: a U-Net
with $128$ base channels, channel multipliers $[1, 2, 2, 2]$, two
residual blocks per resolution, four attention heads with $64$
channels each, attention at the $16{\times}16$ resolution, dropout
$0.1$; trained for $400{,}000$ steps with Adam at learning rate
$2 \times 10^{-4}$, $5{,}000$-step linear warmup, gradient-norm
clipping at $1.0$, batch size $128$, EMA on model weights with decay
$0.9999$, and image normalization to $[-1, 1]$. For
\Cref{fig:cifar10} we sample trajectories from a fixed noise seed
$x_0 \sim \mathcal{N}(0, I)$, integrated forward to $t = 1$ with
\textsc{dopri5}; intermediate columns are uniformly sampled from this
trajectory and rescaled from $[-1, 1]$ to $[0, 1]$ for display. The
six panels of \Cref{fig:cifar10} show samples from each variant at
the end of training. OT-CFM
(\Cref{fig:cifar10_otcfm}) and OT-Harmonic at $\omega = 0.001$
(\Cref{fig:cifar10_harmonic_small}) produce visually
indistinguishable samples, consistent with the $\omega \to 0$ limit
in which the harmonic flow recovers OT-CFM. As $\omega$ increases,
samples grow progressively noisier: $\omega = 1$
(\Cref{fig:cifar10_harmonic_one}) shows mild loss of detail, and
$\omega = \pi/2$ (\Cref{fig:cifar10_harmonic_pi2}) is qualitatively
similar to OT-SI (\Cref{fig:cifar10_otsi}), again matching the
theoretical correspondence at the upper limit. OT-Aniso
(\Cref{fig:cifar10_otaniso}) yields samples comparable to OT-SI
rather than to OT-CFM, suggesting that the data-fit anisotropic
potential does not exploit useful low-dimensional structure on this
benchmark.

\begin{figure}[t]
\centering
\begin{subfigure}[t]{0.33\linewidth}
  \centering
  \includegraphics[width=\linewidth]{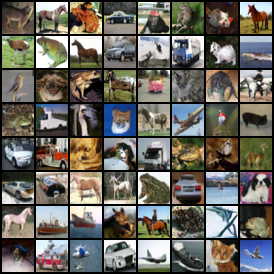}
  \caption{OT-CFM}
  \label{fig:cifar10_otcfm}
\end{subfigure}\hfill
\begin{subfigure}[t]{0.33\linewidth}
  \centering
  \includegraphics[width=\linewidth]{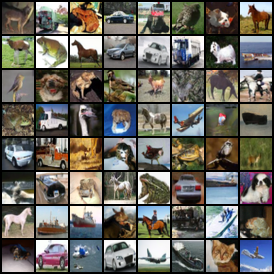}
  \caption{OT-Harmonic, $\omega = 0.001$}
  \label{fig:cifar10_harmonic_small}
\end{subfigure}\hfill
\begin{subfigure}[t]{0.33\linewidth}
  \centering
  \includegraphics[width=\linewidth]{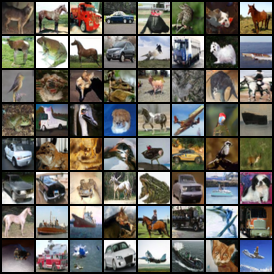}
  \caption{OT-Harmonic, $\omega = 1$}
  \label{fig:cifar10_harmonic_one}
\end{subfigure}

\vspace{0.5em}

\begin{subfigure}[t]{0.33\linewidth}
  \centering
  \includegraphics[width=\linewidth]{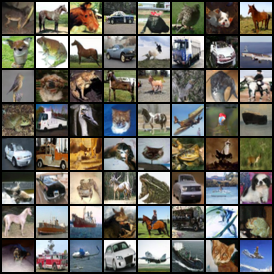}
  \caption{OT-Harmonic, $\omega = \pi/2$}
  \label{fig:cifar10_harmonic_pi2}
\end{subfigure}\hfill
\begin{subfigure}[t]{0.33\linewidth}
  \centering
  \includegraphics[width=\linewidth]{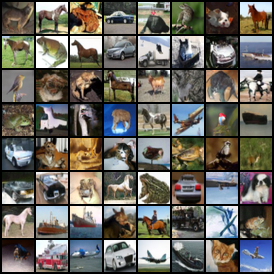}
  \caption{OT-SI}
  \label{fig:cifar10_otsi}
\end{subfigure}\hfill
\begin{subfigure}[t]{0.33\linewidth}
  \centering
  \includegraphics[width=\linewidth]{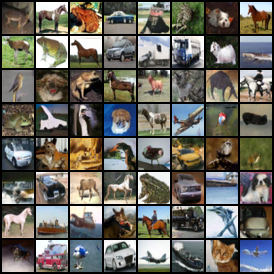}
  \caption{OT-Aniso}
  \label{fig:cifar10_otaniso}
\end{subfigure}
\caption{Uncurated CIFAR-10 samples from each flow-matching variant
after $400{,}000$ training steps, generated with EMA weights and the
\textsc{dopri5} adaptive solver from a fixed noise seed
$x_0 \sim \mathcal{N}(0, I)$. The harmonic family interpolates
between OT-CFM ($\omega \to 0$) and OT-SI ($\omega = \pi/2$) as
$\omega$ increases. OT-Aniso uses the data-fit anisotropic harmonic
potential described in \Cref{app:exp_setup_cifar}.}
\label{fig:cifar10}
\end{figure}

\paragraph{Anisotropic variant.}
The OT-Aniso row fits an anisotropic harmonic potential
$A = R^\top \mathrm{diag}(\omega_1^2,\ldots,\omega_d^2)\, R$ on the
flat image dimension $d = 3 \times 32 \times 32 = 3{,}072$ from a
single fit batch of $10$ training batches ($1{,}280$ images at batch
size $128$). Since the fit batch ($N = 1{,}280$) is smaller than
$d$, $R$ is computed by thin SVD on the centered fit batch: the
$1{,}280$ data-subspace eigenvectors are taken from the
right-singular vectors, and the remaining $d - N = 1{,}792$
null-space directions are completed by orthogonal random vectors
projected against the data subspace. Per-direction frequencies are
assigned linearly: data-subspace directions receive $\omega_k$
uniformly spaced in $[\omega_{\min}, \omega_{\max}]$ with the
highest-variance PC at $\omega_{\min} = 0.8$ and the lowest-variance
data PC at $\omega_{\max} = 1.6$, while the $1{,}792$ null-space
directions all receive $\omega_{\max}$ since no variance evidence
exists for them. The eigenvalue constraint $\omega_k \in (0, \pi)$
holds automatically since $\omega_{\max} = 1.6 < \pi$. The OT
coupling is solved on the anisotropic action cost $c_A$ rather than
squared-Euclidean. $A$ is precomputed once at the start of training
and frozen thereafter.

\subsection{Compute and reproducibility}
\label{app:exp_setup_compute}

All experiments were run on NVIDIA H100 GPUs (80 GB). Code,
configurations, and pretrained checkpoints will be released at the
project URL upon publication.
\section*{Limitations}

We discuss several limitations of the present work.

\paragraph{Restricted Lagrangian class.}
Our results cover Lagrangians of the canonical form \(\mathcal{L}(x,v,t)=K(v)-V(x)\), where \(K\) is a positive-definite quadratic form and \(V\) is sufficiently regular and controlled by the kinetic term. This setting is sufficient to derive the static--dynamic equivalence. However, a simulation-free method requires the least-action trajectories to admit closed-form expressions. Thus, the approach does not extend trivially to general Lagrangians; in that case, one would typically need to solve the Euler--Lagrange boundary-value problem numerically, thereby sacrificing the simulation-free benefits of the current framework. Developing efficient batched solvers for computing these least-action trajectories is therefore a natural direction for future work.

\paragraph{The potential is prescribed, not learned.}
Throughout the paper we treat the potential $V$ (or the matrix $A$ in
the anisotropic case) as a hyperparameter rather than a learnable
component. The harmonic frequency $\omega$ is set by hand, and the
anisotropic potential $A$ is fit once from a small batch of empirical
variances and frozen for the remainder of training. Our CIFAR-10
results (\Cref{tab:fid_nfe}) suggest that this fit-once-and-freeze
strategy is the dominant limitation of the anisotropic variant on
high-dimensional natural images: when the data covariance carries
little geometric signal beyond a near-isotropic spectrum, the
prescribed $A$ provides no meaningful inductive bias. Jointly learning
$V$ with the velocity field, in the spirit of energy-based models
(\Cref{app:ebm}), is a natural extension that we leave to
future work.

\paragraph{No single Lagrangian dominates.}
The harmonic family is not uniformly dominant over its endpoints:
$\omega \to 0$ recovers OT-CFM and is the strongest setting on
CIFAR-10, while $\omega = 1$ improves on both baselines on 2D
synthetic targets and Multiome single-cell interpolation. This is a
structural feature of the framework rather than a defect --- a
Lagrangian encodes an inductive bias, and no fixed bias is optimal
across all data geometries --- but it does mean that the choice of
Lagrangian is itself a modeling decision, analogous to the choice of
kernel in kernel methods or architecture in deep learning. A
principled procedure for selecting $\omega$ (or, more generally, the
potential $V$) from data, perhaps by joint learning as discussed in
\Cref{app:ebm}, would strengthen the practical applicability
of the framework.

\paragraph{Empirical scope.}
Our experiments span 2D synthetic targets, low-dimensional single-cell embeddings, and CIFAR-10 at resolution $32 \times 32$. We have not evaluated on higher-resolution natural images, text-to-image generation, video, or molecular tasks, and the comparative behavior
of different Lagrangians at larger scales remains open. The CIFAR-10
training budget of $400{,}000$ steps is shorter than the run reported
in \citep{tong2024improving} ($500{,}000$ steps); we use this budget uniformly across methods for a controlled comparison within our compute budget, but absolute FID values are correspondingly higher than the longer-run baseline.

\paragraph{Deterministic dynamics only.}
The framework as developed selects probability paths through the
deterministic continuity equation, which makes it a flow-matching
construction rather than a diffusion or Schr\"{o}dinger-bridge
construction. The trigonometric VP path of \Cref{ex:harmonic_to_trig} matches the marginal of the corresponding diffusion model but not the full stochastic forward process, and stochastic Lagrangian extensions --- replacing the continuity equation with a Fokker--Planck equation, in
the spirit of \Cref{app:schrodinger} --- are left to future
work.



\end{document}